\documentclass[twoside,11pt]{article}
\usepackage{myjmlr} 

\usepackage{subfigure}
\usepackage{amsmath}
\usepackage{graphicx}
\usepackage{array}

\newtheorem{op}[theorem]{Optimization Problem} 
\newcommand{\valpha}{{\boldsymbol{\alpha}}}

\newcommand{\vepsilon}{{\boldsymbol{\epsilon}}}
\newcommand{\vzero}{\mathbf{0}}
\renewcommand{\v}[1]{\ensuremath{\mathbf{#1}}}

\begin{document}

\title{Security Analysis of Online Centroid Anomaly Detection}



\author{\name Marius Kloft\thanks{Also at Machine Learning Group, Technische Universität Berlin, Franklinstr. 28/29, FR 6-9, 10587 Berlin, Germany.}
	     \email mkloft@cs.berkeley.edu \\
	     \addr Computer Science Division\\
       \addr University of California\\
             Berkeley, CA 94720-1758, USA
       \AND
       \name Pavel Laskov%\footnote{Also at Fraunhofer Institute FIRST, Kekuléstr. 7, 12489 Berlin, Germany} 
       \email pavel.laskov@uni-tuebingen.de \\
       \addr Wilhelm-Schickard Institute for Computer Science\\
       \addr University of T\"ubingen \\
             Sand 1, 72076 T\"ubingen, Germany
}

%\editor{}


\maketitle

\begin{abstract}
  Security issues are crucial in a number of machine learning
  applications, especially in scenarios dealing with human activity
  rather than natural phenomena (e.g., information ranking, spam
  detection, malware detection, etc.). It is to be expected in such
  cases that learning algorithms will have to deal with manipulated
  data aimed at hampering decision making. Although some previous work addressed the handling of
  malicious data in the context of supervised learning, very little is
  known about the behavior of anomaly detection methods in such
  scenarios. In this contribution,\footnote{A preliminary version of this paper 
  appears in AISTATS 2010, JMLR Workshop and Conference Proceedings, 2010.} we analyze the performance of a
  particular method -- online centroid anomaly detection -- in the
  presence of adversarial noise.  Our analysis addresses the following
  security-related issues: formalization of learning and attack
  processes, derivation of an optimal attack, analysis of its
  efficiency and constraints. We derive bounds on the effectiveness of
  a poisoning attack against centroid anomaly under different
  conditions: bounded and unbounded percentage of traffic, and bounded false positive
  rate. Our bounds show that whereas a poisoning attack can be
  effectively staged in the unconstrained case, it can be made
  arbitrarily difficult (a strict upper bound on the attacker's gain) if
  external constraints are properly used. Our experimental evaluation
  carried out on real HTTP and exploit traces confirms the tightness of
  our theoretical bounds and practicality of our protection
  mechanisms.
\end{abstract}
\vspace*{\baselineskip}

\section{Introduction}
\label{sec:intro}

Machine learning methods have been instrumental in enabling numerous
novel data analysis applications. 
Currently
indispensable technologies such as
object recognition, user preference
analysis, spam filtering -- to name only a few -- all rely on accurate analysis of massive
amounts of data. Unfortunately, the increasing \emph{use} of machine
learning methods brings about a threat of their \emph{abuse}. A
convincing example of this phenomenon are emails that bypass spam
protection tools. Abuse of machine learning can take on various
forms. A malicious party may affect the training data, for
example, when it is gathered from a real operation of a system and
cannot be manually verified. Another possibility is to manipulate
objects observed by a deployed learning system so as to bias its
decisions in favor of an attacker. Yet another way to defeat a
learning system is to send a large amount of nonsense data in order to
produce an unacceptable number of false alarms and hence force a
system's operator to turn it off. Manipulation of a
learning system may thus range from  simple cheating to complete
disruption of its operations.

% The increasing professionalization of attacks against computer systems
% poses a major challenge to information security.  Indiscriminate
% attacks in form of virus and worm outbreaks are making way to a
% targeted and stealthy penetration of extremely sensitive sites.  A
% wide-scale deployment of powerful evasion techniques, such as
% encryption, obfuscation and polymorphism, is manifested in an
% exploding diversity of malicious software observed by security
% experts. This makes a use of pattern matching techniques traditionally
% deployed in end-user security products all but pointless, since a
% manual update of characteristic patterns -- also known as \emph{attack
%   signatures} -- cannot keep up with the pace of ``attack
% technology''.

A potential insecurity of machine learning methods stems from the fact
that they are usually not designed with adversarial input in
mind. Starting from the mainstream computational learning theory
\citep{Vap98,SchSmo02}, a prevalent assumption is that training and
test data are generated from the same, fixed but unknown, probability
distribution. This assumption obviously does not hold for adversarial
scenarios. Furthermore, even the recent work on learning with
differing training and test distributions \citep{SugKraMue07} is not
necessarily appropriate for adversarial input, as in the latter case one must account for a
specific worst-case difference.

The most important application field in which robustness of learning
algorithms against adversarial input is crucial is computer
security. Modern security infrastructures are facing an increasing
professionalization of attacks motivated by monetary profit.  A
wide-scale deployment of insidious evasion techniques, such as
encryption, obfuscation and polymorphism, is manifested in an
exploding diversity of malicious software observed by security
experts. Machine learning methods offer a powerful tool to counter a
rapid evolution of security threats. For example, anomaly detection
can identify unusual events that potentially contain novel, previously
unseen exploits
\citep{WanSto04,RieLas06,WanParSto06,RieLas07}. Another typical
application of learning methods is automatic signature generation
which drastically reduces the time needed for a production and
deployment of attack signatures \citep{NewKarSon06,LiSanCheKaoCha06}.
Machine learning methods can also help researchers to better
understand the design of malicious software by using classification or
clustering techniques together with special malware acquisition and
monitoring tools \citep{Bailey:07:ACA,RieHolWilDueLas08}.

In order for machine learning methods to be successful in security
applications -- and in general in any application where adversarial
input may be encountered -- they should be equipped with
countermeasures against potential attacks. The current understanding
of security properties of learning algorithms is rather
patchy. Earlier work in the PAC-framework has addressed some scenarios
in which training data is deliberately corrupt
\citep{AngLai88,Lit88,KeaLi93,Aue97,BshEirKus99}. These results,
however, are not connected to modern learning algorithms used in
classification, regression and anomaly detection problems. On the
other hand, several examples of effective attacks have been
demonstrated in the context of specific security and spam detection
applications
\citep{LowMee05,FogShaPerKolLee06,FogLee06,PerDagLeeFogSha06,NewKarSon06,NelBarChiJosRubSaiSutTygXia08},
which has motivated a recent work on taxonomization of such attacks
\citep{BarNelSeaJosTyg06,BarChiJosNelRubSaiTyg08}.
However, it remains largely unclear whether machine learning
methods can be protected against adversarial impact.

We believe that an unequivocal answer to the problem of ``security of
machine learning'' does not exist. The security properties cannot be
established experimentally, as the notion of security deals with
events that do not just happen on average but rather only potentially may
happen. Hence, a theoretical analysis of machine learning algorithms
for adversarial scenarios is indispensable. It is hard to imagine,
however, that such analysis can offer meaningful results for any
attack and any circumstances. Hence, to be a useful guide for
practical applications of machine learning in adversarial
environments, such analysis must address \emph{specific attacks against
specific learning algorithms}. This is precisely the approach followed
in this contribution.

The main focus of our work is a security analysis of online centroid
anomaly detection against the so-called ``poisoning'' attacks. The
centroid anomaly detection is a very simple method
which has been widely used in computer security applications
\citep[e.g.,][]{ForHofSomLon96,WarForPea99,WanSto04,RieLas06,WanParSto06,RieLas07}.
In the learning phase, centroid anomaly detection computes the mean of
all training data points:
$$
\v c = \frac{1}{n} \sum_{i=1}^n \v x_i.
$$
Detection is carried out by computing the distance of a new
example $\v x$ from the centroid $\v c$ and comparing it with an
appropriate threshold:
$$
f(\v x) =
\begin{cases}
  1, \quad \text{if $||\v x - \v c|| > \theta$} \\
  0, \quad \text{otherwise.}
\end{cases}
$$
Notice that all operations can be carried out using kernel functions 
-- a standard trick known since the kernel PCA
\citep{SchSmoMue98,ShaCri04} -- which substantially increases the
discriminative power of this method.

More often than not, anomaly detection algorithms are deployed in
non-stationary environments, hence need to be regularly re-trained. In
the extreme case, an algorithm learns online by updating its
hypothesis after every data point it has received. Since the data is
fed into the learning phase without any verification, this opens a
possibility for an adversary to force a learning algorithm to learn a
representation suitable for an attacker. One particular kind of attack
is the so-called ``poisoning'' in which specially crafted data points
are injected so as to cause a hypothesis function to  misclassify a
given malicious point as benign. This attack makes sense when an
attacker does not have ``write'' permission to the training data, hence
cannot manipulate it directly. Therefore, his goal is to trick an
algorithm by merely using an ``append'' permission, by sending new
data. 

The poisoning attack against online centroid anomaly detection has
been considered by \cite{NelJos06} for the case of infinite window,
i.e., when a learning algorithm memorizes all data seen so far. Their
main result was surprisingly optimistic: it was shown that the number
of attack data points to be injected grows exponentially as a function
of the impact over a learned hypothesis. However, the assumption of
an infinite window also hinders the ability of a learning algorithm to
adjust to legitimate changes in the data distribution. 

As a main contribution of this work, we present the security analysis
of online centroid anomaly detection for the finite window case,
i.e., when only a fixed number of data points can be used at any time
to form a hypothesis. We show that, in this case, an attacker can
easily compromise a learning algorithm by using only a linear amount
of injected data unless additional constraints are imposed. As a
further contribution, we analyze the algorithm under two additional
constraints on the attacker's part: (a) the fraction of the traffic
controlled by an attacker is bounded by $\nu$, and (b) the false
positive rate induced by an attack is bounded by $\alpha$. Both of
such constraints can be motivated by an operational practice of
anomaly detection systems. Overall, we significantly extend the
analysis of \cite{NelJos06} by considering a more realistic learning
scenario, explicitly treating potential constraints on the attacker's
part and providing tighter bounds.

The methodology of our analysis follows the following framework, which
we believe can be used for a \emph{quantitative security analysis}  of
learning algorithms \citep{LasKlo09}:

\begin{enumerate}

\item \emph{Axiomatic formalization of the learning and attack
    processes}. The first step in the analysis is to formally specify
  the learning and attack processes. Such formalization includes
  definitions of data sources and objective (risk) functions used by
  each party, as well as the attack goal. It specifies the knowledge
  available to an attacker, i.e., whether he knows an algorithm, its
  parameters and internal state, and which data he can potentially
  manipulate.

\item \emph{Specification of an attacker's constraints}. Potential
  constraints on the attacker's part may include: percentage of
  traffic under his control, amount of additional data to be injected,
  an upper bound on the norm of manipulated part, a maximal allowable
  false-positive rate (in case an attack must stealthy), etc. Such
  constraints must be incorporated into the axiomatic formalization.  

\item \emph{Investigation of an optimal attack policy}. Given a formal
  description of the problem and constraints, an optimal attack policy
  must be investigated. Such policy may be long-term, i.e., over
  multiple attack iteration, as well as short-term, for a single
  iteration. Investigation can be carried out either as a formal proof
  or numerically, by casting the search for an attack policy as an
  optimization problem.

\item \emph{Bounding of an attacker's gain under an optimal policy}. The
  ultimate goal of our analysis is to quantify an attacker's gain or
  effort under his optimal policy. Such analysis may take different
  forms, for example calculation of the probability for an attack to
  succeed, estimation of the required number of attack iterations,
  calculation of the geometric impact of an attack (a shift towards an
  insecure state), etc.

\end{enumerate}

Organization of this paper reflects the main steps of the proposed
methodology. In a preliminary Section~\ref{sec:prelims} the models of
the learning and the attack processes are introduced. The analytical
part is arranged in two sections as follows.  Section~\ref{sec:full} addresses
the steps (1), (3) and (4) under an assumption that an attacker has
full control of the network traffic. Section~\ref{sec:limited}
introduces an additional assumption that attacker's control is limited
to a certain fixed fraction of network traffic, as required in step
(2). Another constraint of the bounded false positive rate is
considered in Section~\ref{sec:defense}. This section also removes a
somewhat unrealistic assumption of Section~\ref{sec:limited} that all
innocuous points are accepted by the algorithm.  The analytic results are
experimentally verified in Section~\ref{sec:ids} on real HTTP data and
attacks used in intrusion detection systems. Some proofs and the
auxiliary technical material are presented in the Appendix.

Before moving on to the detailed presentation of our analysis, it may
be instructive to discuss the place of a poisoning attack in the
overall attack taxonomy and practical implication of its
assumptions. For two-class learning problems, attacks against learning
algorithms can be generally classified according to the following two
criteria (the terminology in the taxonomy of \cite{BarNelSeaJosTyg06}
is given in brackets):
\begin{itemize}
\item whether an attack is staged during the training (causative) or the
  deployment of an algorithm (causative/exploratory), or
\item whether an attack attempts to increase the false negative or the
  false positive rate at the deployment stage
  (integrity/availability).
\end{itemize}
The poisoning attack addressed in our work can be classified as a
causative integrity attack. This scenario is quite natural, e.g., in web
application scenarios in which the data on a server can be assumed
secure but the injection of adversarial data cannot be easily
prevented. Other common attack types are a mimicry attack --
alteration of malicious data to resemble innocuous data (an exploratory
integrity attack), or a ``red herring'' attack -- sending of junk data
that causes false alarms (an exploratory availability attack). Attacks
of the latter two kinds are beyond the scope of our investigation.

As a final remark, we must consider the extent to which the attacker
is familiar with the learning algorithm and trained model. 
One of the key principles of computer security, known
as \emph{Kerckhoff's principle}, is that the robustness of any security
instrument must not depend on keeping its operational functionality
secret. Similar to modern cryptographic methods, we must assume
that the attacker knows which machine learning algorithm is
deployed and how it operates (he can even use machine learning to
reverse engineer deployed classifiers, as shown by
\cite{LowMee05a}). A more serious difficulty on the attacker's part
may be to get hold of the training data or of the particular learned
model. In the case of anomaly detection, it is relatively easy for an
attacker to retrieve a learned model: it suffices to sniff on the same
application that is protected by an algorithm to get approximately the
same innocuous data the algorithm is trained on. Hence, we will
assume that an attacker has precise knowledge of
the trained model at any time during the attack.

\section{Learning and Attack Models}
\label{sec:prelims}

Before proceeding with the analysis, we first present the precise
models of the learning and the attack processes. Our focus on anomaly
detection is motivated by its ability to detect potentially novel
attacks, a crucial demand of modern information security.

\subsection{Centroid Anomaly Detection}
\label{sec:centroid}

Given the data set $X = \{\v x_1, \ldots, \v x_n\}$, the goal of
anomaly detection (also often referred to as ``novelty detection'') is
to determine whether an example $\v x$ is unlikely to have been
generated by the same distribution as the set $X$. A natural way to
perform anomaly detection is to estimate a probability density
function of the distribution from which the set $X$ was drawn and flag
$\v x$ as anomalous if it comes from a region with low density.  In
general, however, density estimation is a difficult problem,
especially in high dimensions. A large amount of data is usually
needed to reliably estimate the density in all regions of the space.  For
anomaly detection, knowing the density in the entire space is
superfluous, as we are only interested in deciding whether a specific
point falls into a ``sparsely populated'' area.  Hence several direct
methods have been proposed for anomaly detection, e.g., one-class SVM
\citep{SchPlaShaSmoWil01}, support vector data description (SVDD)
\citep{TaxDui99,TaxDui99a}, and density level set estimation
\citep{Pol95,Tsy97,SteHusSco05}. A comprehensive survey of anomaly
detection techniques can be found in \cite{MarSin03a,MarSin03b}.

In the centroid anomaly detection, a Euclidean distance from an
empirical mean of the data is used as a measure of anomality: 
$$
f(\v x) = ||\v x - \frac{1}{n} \sum_{i=1}^{n} \v x_i ||.
$$
If a hard decision is desired instead of a soft anomality score,
the data point is considered anomalous if its anomaly score exceeds a
fixed threshold $r$.

Centroid anomaly detection can be seen as a special case for the SVDD
with outlier fraction $\eta = 1$ and of the Parzen window density
estimator \citep{Par62} with the Gaussian
kernel function $k(\v x,\v y) = \frac{1}{\sqrt{ 2 \pi}} \exp (- \frac{1}{2} \v x\cdot\v y)$.
Despite its straightforwardness, 
a centroid model can represent arbitrary complex density level sets using a kernel mapping \citep{SchSmo02,MueMikRaeTsuSch01} 
(see Fig.~\ref{fig:rbf-centroid}).

\begin{figure}
  \centering
  \includegraphics[width=0.4\textwidth]{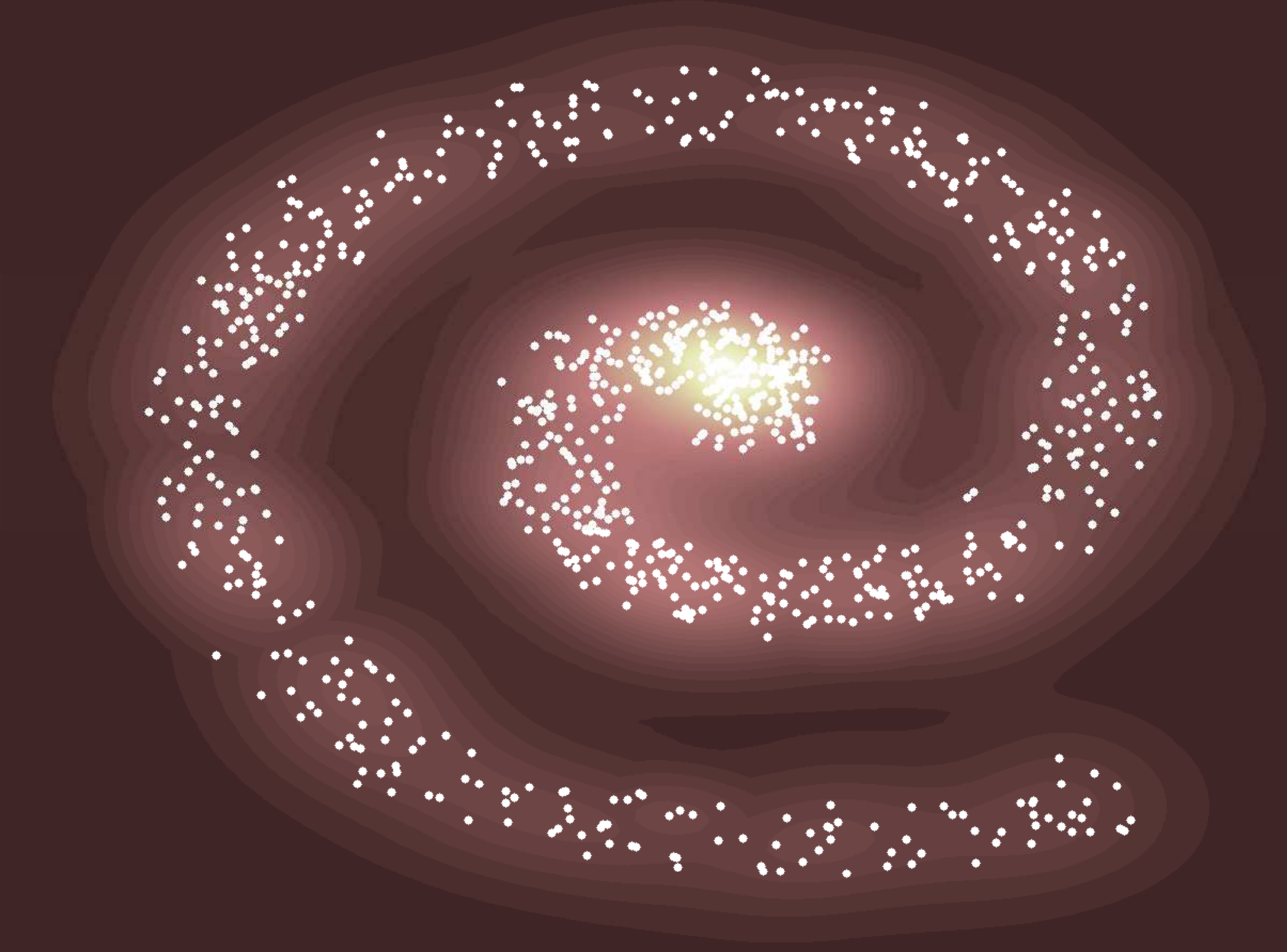}
  \caption{Illustration of the density level estimation using a
    centroid model with a non-linear kernel.}
  \label{fig:rbf-centroid}
\end{figure}

%Centroid anomaly detection has
%recently re-surfaced under the notions of centering in feature spaces
%\citep{ShaCri04} or quarter-sphere SVM \citep{LasSchKot04}. 
It has been
successfully used in a variety of anomaly detection applications such
as intrusion detection
\citep{HofForSom98,YeuCho02,LasSchKot04,WanSto04,RieLas06,WanParSto06,RieLas07},
wireless sensor networks \citep{RajLecPalBez07} and jet engine
vibration data analysis \citep{NaiTowCarKinCowTar99}. It can be shown
(cf. \cite{ShaCri04}, Section 4.1) that even in high-dimensional
spaces induced by nonlinear feature maps, the empirical estimator of
the center of mass of the data is stable and the radius of a sphere
anchored at the center of mass is related to a level set of the
corresponding probability density.

\subsection{Online Anomaly Detection}
\label{sec:online}

The majority of anomaly detection applications have to deal with
non-stationary data. This is especially typical for computer
security, as usually the processes being monitored change over time: e.g.,
network traffic profile is strongly influenced by the time of the day and
system call sequences depend on the applications running on a
computer. Hence the model of normality constructed by anomaly
detection algorithms usually needs to be updated during their
operations. In the extreme case, such an update can be performed after
the arrival of each data point resulting in the online operation.
Obviously, re-training the model from scratch every time is
computationally infeasible; however, incorporation of new data
points and the removal of irrelevant ones can be done with acceptable
effort \citep{LasGehKruMue06}.

For the centroid anomaly detection, re-calculation of the center of
mass is straightforward and requires $O(1)$ work. If all examples are
``memorized'', i.e., the index $n$ is growing with the arrival of each
example, the index $n$ is incremented for every
new data point, and the update is computed as\footnote{The update
  formula can be generalized to $\v c' = \v c + \frac{\kappa}{n}(\v x
  - \v x_i)$, with fixed $\kappa\geq 1$. The bounds in the analysis
  change only by a constant factor, which is negligible.}
\begin{equation}
\label{eq:update-infinite}
\v c' = \left (1 - \frac{1}{n} \right) \v c + \frac{1}{n} \v x.
\end{equation}  
For the finite horizon, i.e. constant $n$, some previous example $\v
x_i$ is replaced by a new one, and the update is performed as
\begin{equation}
\label{eq:update-finite}
\v c' = \v c + \frac{1}{n}(\v x - \v x_i).
\end{equation}

Various strategies can be used to determine the ``least relevant''
point $\v x_i$ to be removed from a working set:
\renewcommand{\descriptionlabel}[1]{\hspace{\labelsep}\textsf{#1:}}
\begin{itemize}
\item[(a)] \textsf{oldest-out}: The point with the oldest timestamp is
  removed.
\item[(b)] \textsf{random-out}: A randomly chosen point is removed.
\item[(c)] \textsf{nearest-out}: The nearest-neighbor of the new point
  $\v x$ is removed.
\item[(d)] \textsf{average-out}: The center of mass is removed.  The
  new center of mass is recalculated as $\v c'= \v c +\frac{1}{n}(\v
  x-\v c)$, which is equivalent to Eq.~(\ref{eq:update-infinite}) with
  constant $n$.
\end{itemize}
The strategies (a)--(c) require the storage of all points in the
working set, whereas the strategy (d) can be implemented by holding
only the center of mass in memory.

\subsection{Poisoning attack}
\label{sec:poisoning}

The goal of a poisoning attack is to force an anomaly detection
algorithm to accept an attack point $\v A$ that lies outside of the
normal ball, i.e., $|| \v A - \v c || > r$. It is assumed that an
attacker knows the anomaly detection algorithm and all the training
data. However, an attacker cannot modify any existing data except for
adding new points. These assumptions model a scenario in which an
attacker can sniff data on the way to a particular host and can send
his own data, while not having write access to that host.  As
illustrated in Fig.~\ref{fig:poisoning}, the poisoning attack attempts
to inject specially crafted points that are accepted as innocuous and
push the center of mass in the direction of an attack point until the
latter appears innocuous.

%Suppose an attack vector $\v a$ lies outside of the normal ball,
%i.e. $|| \v a - \v c || > r$. Naturally, it will be reported as
%anomalous and rejected by an intrusion detection system.  The idea of
%the poisoning attack is to force a system to accept an attack vector
%$\v a$ by injecting specially crafted data points into the training
%data so as to cause the center of mass to drift into a desired
%direction. After a sufficient number of malicious points have been
%injected an attack vector falls within the normal ball. Without loss
%of generality, we assume that the initial center of mass $\v c_0 =
%0$. The poisoning attack is illustrated in Fig.~\ref{fig:poisoning}.

\begin{figure}
\centering
\includegraphics[width=0.50\textwidth]{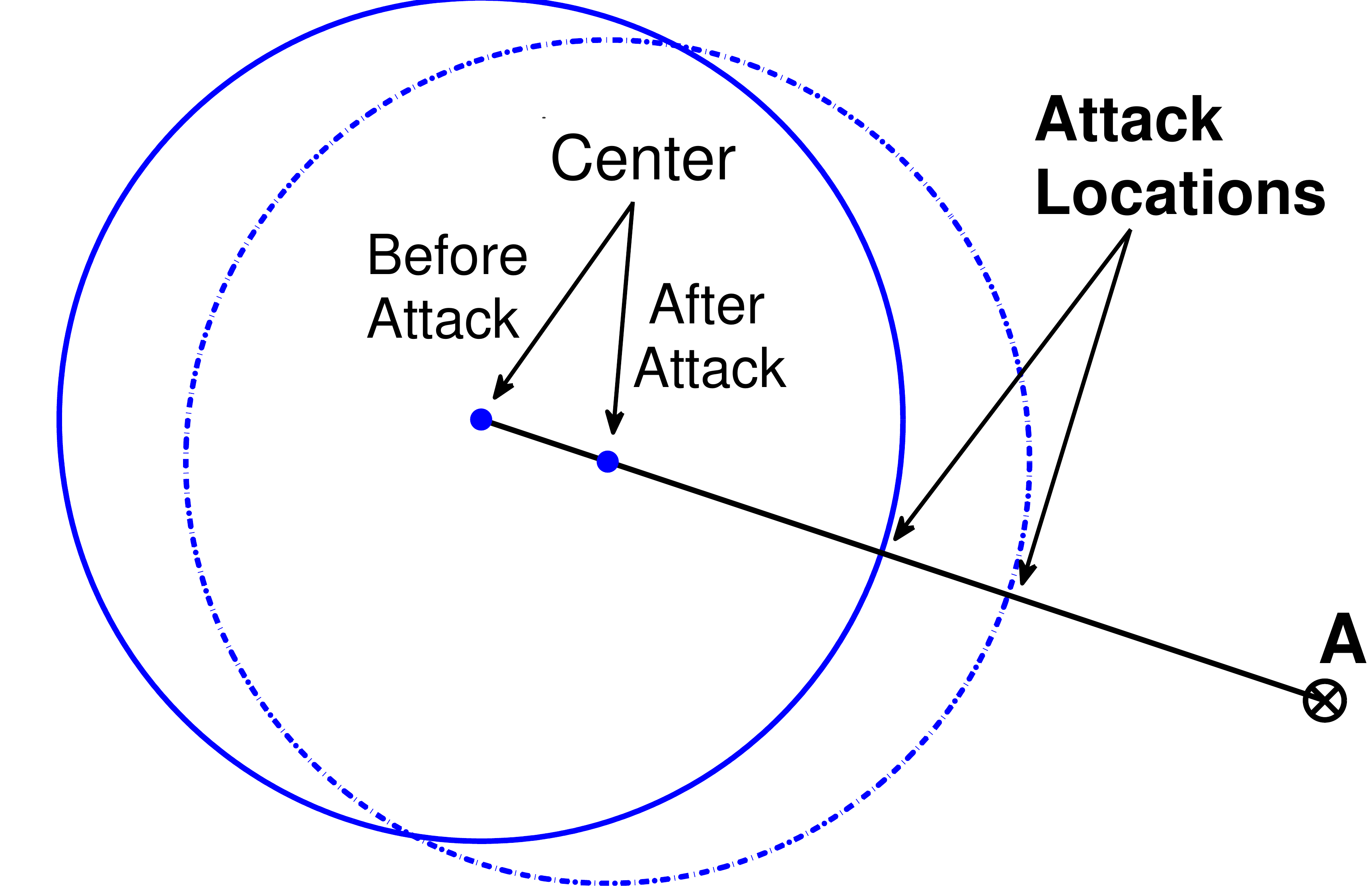}
\caption{Illustration of a poisoning attack. By iteratively inserting malicious training points
an attacker can gradually corrupt ``drag'' the centroid into a
direction of an attack. }
\label{fig:poisoning}
\end{figure}

What points should be used by an attacker in order to subvert online
anomaly detection? Intuitively one can expect that the optimal
one-step displacement of the center of mass is achieved by placing
attack point $\v x_i$ at the line connecting $\v c$ and $\v A$ such
that $||\v x_i - \v c|| = r$. A formal proof of the optimality of such
strategy and estimation of its efficiency constitutes the main
objective of security analysis of online anomaly detection.

In order to quantify the effectiveness of a poisoning attack, we
define the $i$-th relative displacement of the center of mass.
%to be $D_i = \frac{( \v c_i - \v c_0) \cdot \v a}{r|| \v a ||}$ (w.l.o.g. we assume that $\v c_0= \v 0$).  
This quantity
measures the relative length of the projection of $\v c_i$ onto the ``attack direction'' $\v a$
in terms of the radius of the normality ball. 

\begin{definition}[Relative displacement]
\hspace{0cm}

 (a)  Let $\v A$ be an attack point and define  by $\v a=\frac{\v A -\v c_0}{||\v A -\v c_0||}$ the according attack direction vector. The \emph{$i$-th relative displacement},
  denoted by $D_i$, is defined as 
  $$D_i = \frac{( \v c_i - \v c_0) \cdot \v a}{r}$$.
  W.l.o.g. we assume that $\v c_0= \v 0$.

(b)
	Attack strategies maximizing the 
  displacement $D_i$ in each iteration $i$ are referred to as \emph{greedy optimal attack
    strategies}.
\end{definition}

%It easily follows from
%the cosine rule that an attack vector is accepted when
%\begin{equation}
%D_i \geq D^{*} = \frac{||\v c_i||^2 + ||\v a||^2 - r^2}{2 r ||\v
%  a||}.
%\label{eq:accepted}
%\end{equation}
%See Fig.~\ref{fig:cosine_rule} for illustration.
%
%\begin{figure}
%\centering
%\includegraphics[width=0.53\textwidth]{cosine_rule}
%\caption{}
%\label{fig:cosine_rule}
%\end{figure}

\section{Attack Effectiveness for Infinite Horizon Centroid Learner}\label{sec:infiniteHorizon}
%\subsection{Attack Effectivess}

The effectiveness of a poisoning attack for an infinite horizon has
been analyzed in \cite{NelJos06}. We provide an alternative proof that
follows the framework proposed in the introduction.
\begin{theorem}\label{th:infinite}
  The $i$-th relative displacement $D_i$ of the online centroid learner with an infinite
  horizon under the poisoning attack is bounded by
  \begin{equation}
    \label{eq:nelson-bound}
    D_i \leq  \, {\rm ln}\left(1+\frac{i}{n}\right) ~ ,
    %\frac{r}{||\v a||} rausgenommen, da falsch
  \end{equation}
  where $i$ is the number of attack points and $n$ the number of
  initial training points.
\end{theorem}

\begin{proof}
  We first determine an optimal attack strategy and then bound the
  attack progress.

  (a) Let $\v A$ be an attack point and denote by $\v a$ the corresponding attack direction vector. Let $\{\v a_i|i\in\mathbb
  N\}$ be adversarial training points. The center of mass at the
  $i$-the iteration is given in the following recursion:
\begin{equation}
\label{eq:recursion-center} 
\v c_{i+1} =  \left(1-\frac{1}{n+i}\right)\v c_i+\frac{1}{n+i}\v
a_{i+1},
\end{equation}
with initial value $\v c_0=0$. By the construction of the poisoning
attack, $||\v a_i-\v c_i || \leq r$, which is equivalent to $\v
a_{i}=\v c_i+ \v b_i$ with $|| \v b_i || \leq r$.
Eq.~(\ref{eq:recursion-center}) can thus be transformed into
$$ 
\v c_{i+1} =  \v c_i+\frac{1}{n+i}\v b_i .
$$
Taking scalar product with $\v a$ and using the definition of a relative
displacement, we obtain:
\begin{equation}
\label{eq:recursion-disp}
  D_{i+1} =  D_i+\frac{1}{n+i} \cdot \frac{\v b_i \cdot \v a}{r} ,
\end{equation}
with $D_0=0$. The right-hand side of the Eq.~(\ref{eq:recursion-disp}) is
clearly maximized under $||\v b_i||\leq 1$ by setting $\v b_i =r \v a$. Thus
the optimal attack is defined by 
\begin{equation}
\v a_i = \v c_i + r \v a.
\label{eq:opt-attack}
\end{equation}
(b) Plugging the optimal strategy $\v b_i = r\v a$ into
Eq~\eqref{eq:recursion-disp}, we have:
$$ 
D_{i+1} =  D_i+\frac{1}{n+i} .
$$
This recursion can be explicitly solved, taking into account that $d_0
= 0$, resulting in:
$$ 
D_i = \sum_{k=1}^i\frac{1}{n+k} = \sum_{k=1}^{n+i}\frac{1}{k}
-\sum_{k=1}^n\frac{1}{k} ~ .
$$
Inserting the upper bound on the harmonic series, $\sum_{k=1}^m \frac{1}{k}=\ln(m)+\epsilon_m$ with $\epsilon_m\geq 0$ into the above formula, and noting that $\epsilon_m$ is monotonically decreasing, we obtain
$$  D_i \leq {\rm ln}(n+i) - {\rm ln}(n) = {\rm ln}\left(\frac{n+i}{n}\right) 
  = {\rm ln}\left(1+\frac{i}{n}\right) ,$$
which completes the proof.
\end{proof}
\label{app:proofs-prelims}

Since the bound in Eq.~(\ref{eq:nelson-bound}) is monotonically
increasing, we can invert it to obtain the estimate of the effort
needed by an attacker to achieve his goal:
$$ i\geq n\cdot\left(\exp \left(D^{*} \right)-1\right) ~ .$$
It can be seen that an effort need to poison a online centroid learner is
\emph{exponential} in terms of the relative displacement of the center
of mass.\footnote{Even constraining a maximum number of online update steps cannot remove the
bound's exponential growth \citep{NelJos06}.}
In other words, an attacker's effort grows prohibitively fast
with respect to the separability of an attack from the innocuous data.
However, this is not surprising since due the infinitely growing training window
the contribution of new points to the computation of the center of mass is steadily decreasing.

%which
%is determined by the attack length $||\v a||$ and the radius $r$ (the
%more compact is the normal data, the smaller radius $r$ can be set for
%a giver false alarm rate).

%the ratio of the
%attack length to the radius of the sphere defining the normal data
%region. 

% If the attack vector is well separated from normal data and a
% small radius can be used without generating too many false positives,
% a online centroid learner with an infinite horizon can be considered secure
% against a poisoning attack.
% \textcolor{red}{Since we measure an attacker's effort by the \emph{relative} displacement of the center of mass
% we from now for simplicity of analysis without loss of generality assume r=1. [An Pavel: Nachfolgende Vorkommen von r müssen noch mit 1 ersetzt werden.}

\section{Poisoning Attack against Finite Horizon Centroid Learner}
\label{sec:full}

As it was shown in Section~\ref{sec:poisoning}, the poisoning attack
is ineffective against online centroid anomaly detection if all points
are kept ``in memory''. Unfortunately, memorizing the points defeats
the main purpose of online algorithms, i.e., their ability to adjust to
non-stationarity\footnote{Once again we remark that the data need not
  be physically stored, hence the memory consumption is not the main
  bottleneck in this case.}. Hence it is important to understand how
the removal of data points from a working set affects the security of
online anomaly detection. For that, the specific removal strategies
presented in Section~\ref{sec:online} must be considered.

It will turn out that for the average- and random-out rules the analysis can be carried out theoretically.
For the nearest-out rule the analysis is more complicated but an optimal attack can be stated as mathematical optimization problem, 
and the attack effectiveness can be analyzed empirically.

\subsection{Poisoning Attack for Average- and Random-out Rules}\label{subsec:average-out}

We begin our analysis with the average-out learner which follows
exactly the same update rule as the infinite-horizon online centroid learner
with the exception that the window size $n$ remains fixed instead of
growing indefinitely (cf. Section~\ref{sec:online}). 
Despite the similarity to the infinite-horizon case, the
result presented in the following theorem is surprisingly pessimistic. 

\begin{theorem}
\label{th:average-out}
  The $i$-th relative displacement $D_i$ of the online centroid learner with the average-out
  update rule under an worst-case optimal poisoning attack is 
  \begin{equation}
    \label{eq:average-bound}
    D_i = \frac{i}{n},
  \end{equation}
  where $i$ is the number of attack points and $n$ is the training window size. 
\end{theorem}

\begin{proof}
The proof is similar to the proof of
Theorem~\ref{th:infinite}. By
explicitly writing out the recurrence between subsequent
displacements, we conclude that the optimal attack is also attained by
placing an attack point on the line connecting $\v c_i$ and $\v a$ at
the edge of the sphere (cf. Eq.~(\ref{eq:opt-attack})):
$$
\v a_i = \v c_i + r \v a.
$$
It follows that the relative displacement under the optimal attack is
$$ 
D_{i+1} =  D_i+\frac{1}{n} .
$$
Since this recurrence is independent of the running index $i$, the
displacement is simply accumulated over each iteration, which yields
the bound of the theorem.  
\end{proof}

One can see, that unlike the logarithmic bound in
Theorem~\ref{th:infinite}, the average-out learner is characterized
by a linear bound on the displacement. As a result, an attacker only
needs a linear amount of injected points -- instead of an exponential
one -- in order to subvert an average-out learner. This cannot be
considered secure.

We obtain a similar result for the random-out removal strategy.

\begin{theorem}
\label{th:random-out}
  For the $i$-th relative displacement $D_i$ of the online centroid learner with the random-out
  update rule under an worst-case optimal poisoning attack it holds
  \begin{equation}
    %\label{eq:random-bound}
    E(D_i) = \frac{i}{n},
  \end{equation}
  where $i$ is the number of attack points, $n$ is the training window size, and the expectation is drawn over the choice of the removed data points. 
\end{theorem}

\begin{proof}
The proof is based on the observation that the random-out rule in 
expectation boils down to average-out, and hence is reminiscent to the proof of Th.~\ref{th:average-out}.
\end{proof}

\subsection{Poisoning Attack for Nearest-out Rule}
\label{subsec:nearest-out}

Let us consider the alternative update strategies mentioned in
Section~\ref{sec:centroid}. The update rule $\v c' = \v c +
\frac{1}{n} (\v x - \v x_{0})$ of the oldest-out strategy is
essentially equivalent to the update rule of the average-out except
that the outgoing center $\v c$ is replaced by the oldest point $\v
x_{0}$. In both cases the point to be removed is fixed in
advance regardless of an attacker's moves, hence the pessimistic
result developed in Section~\ref{subsec:average-out} remains valid for
this case. On average, the random-out update strategy is -- despite its
nondeterministic nature -- equivalent to the average-out strategy.
Hence, it also cannot be considered secure against a poisoning attack.

One might expect that the nearest-out strategy poses a stronger
challenge to an attacker, as it tries to keep as much of a working set
diversity as possible by retaining the most similar data to a new
point. It turns out, however, that even this strategy can be broken
with a feasible amount of work if an attacker follows a greedy optimal
strategy. The latter is a subject of our investigation in this
section.

\subsubsection{An optimal attack}
\label{subsec:optimal-attack}

Our investigation focuses on a \emph{greedy} optimal attack, i.e., an
attack that provides a maximal gain for an attacker in a single
iteration.
For the infinite-horizon learner (and hence also for the
average-out learner, as it uses the same recurrence in a proof), it is
possible to show that the optimal attack yields the maximum gain for
the entire sequence of attack iterations. For the nearest-out learner,
it is hard to analyze a full sequence of attack iterations, hence we
limit our analysis to a single-iteration gain. Empirically, even a
greedy optimal attack turns out to be effective.

To construct a greedy optimal attack, it suffices to determine for
each point $\v x_i$ the location of an optimal attack point $\v x_i^{*}$ to
replace $\v x_i$. This can be formulated as the following optimization
problem:
%
% Tricky: manual advancement of a counter with a subsequent label is
% needed here!

\begin{op}[greedy optimal attack]\label{op:greedy}
\renewcommand{\minalignsep}{6pt}
\refstepcounter{equation}
\label{eq:greedy-raw} 
\begin{alignat}{3}
  \{\v x_i^{*}, f_i\} \; =\;\; && \max_{\v x}\;\; &&& (\v x-\v x_i)\cdot\v
  a \tag{\theequation.a}\\
  &&\text{s.t.}\;\; &&& \Vert\v x-\v x_i\Vert \leq \Vert\v x-\v x_j\Vert,
  \quad
  \forall j=1,...,n \label{eq:greedy-raw-cell}\tag{\theequation.b}\\
  &&&&& \Vert\v x - \textstyle\frac{1}{n}\sum_{j=1}^n\v x_j \Vert \leq
  r.\label{eq:greedy-raw-ball}\tag{\theequation.c}
\end{alignat}
\end{op}
The objective of the optimization problem \ref{op:greedy}
reflects an attacker's goal of maximizing the projection of $\v x - \v
x_i$ onto the attack direction vector $\v a$. The constraint
(\ref{eq:greedy-raw-cell}) specifies the condition that the point $\v
x_i$ is the nearest neighbor of $\v x$ (i.e., $\v x$ falls into a
\emph{Voronoi cell} induced by $\v x_i$). The constraint
(\ref{eq:greedy-raw-ball}), when active, enforces that no solution
lies outside of the sphere. Hence the geometric intuition behind an
optimal attack, illustrated in Figure~\ref{fig:nearestout}, is to
replace some point with an attack point placed at the ``corner'' of
the former's Voronoi cell (including possibly a round boundary of the
centroid) that provides a highest displacement of the center in the
attack point's direction.
\begin{figure}
\centering
\includegraphics[width=0.6\textwidth]{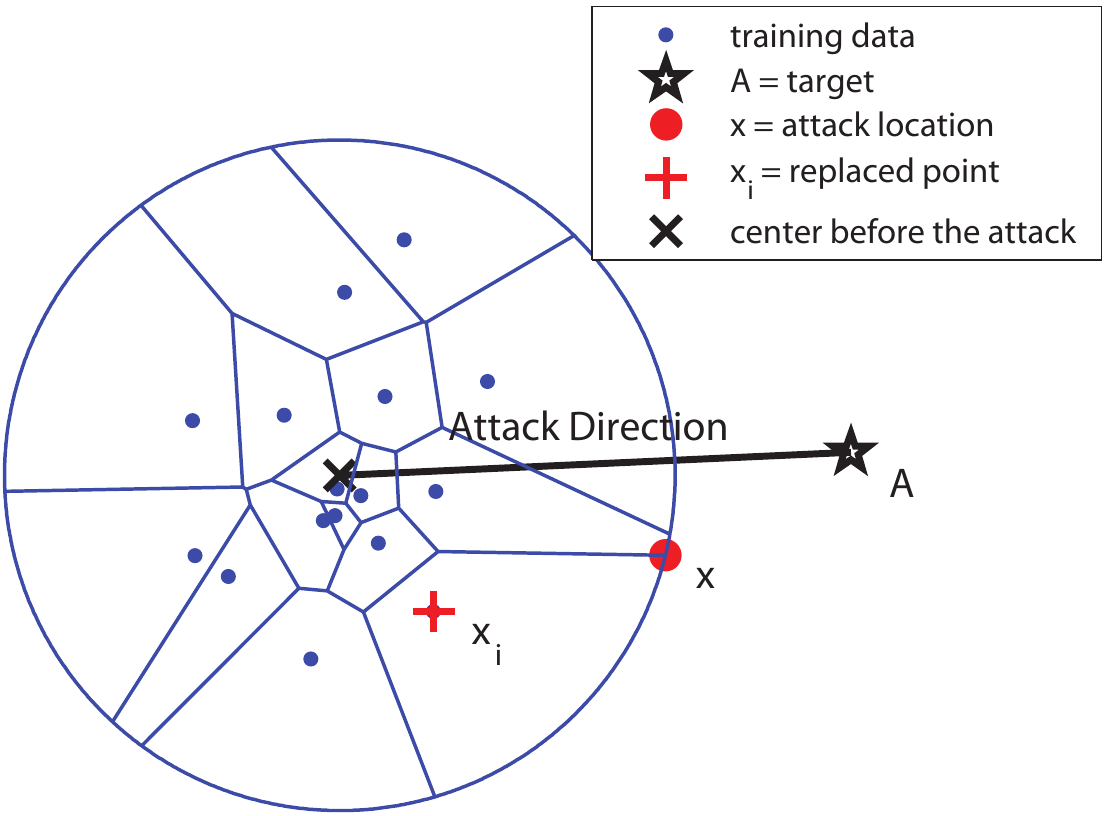}
\caption{The geometry of a poisoning attack for the nearest-out rule. An optimal attack 
is achieved at the boundary of a Voronoi cell.}
\label{fig:nearestout}
\end{figure}

The maximization of Eq.~(\ref{eq:greedy-raw}) over all points in a
current working set yields the index of the point to be replaced by an
attacker:
\begin{equation}
  \label{eq:outer-loop}
  \alpha = \text{argmax}\,_{i\in{1,...,n}} \;f_i
\end{equation}

By plugging the definition of a Euclidean norm into the inner
optimization problem (\ref{eq:greedy-raw}) and multiplying out the
quadratic constraints, all but one norm constraints reduce to simpler linear constraints:
\refstepcounter{equation}
\label{eq:greedy-final} 
\begin{alignat}{3}
  \{\v x_i^{*}, f_i\} \; =\;\; && \max_{\v x}\;\; &&& (\v x-\v
  x_i)\cdot\v
  a \tag{\theequation.a}\\
  &&\text{s.t.}\;\; &&& 2(\v x_j-\v x_i)\cdot \v x \leq \v x_j\cdot \v
  x_j - \v x_i\cdot \v x_i , \quad
  \forall j=1,...,n \label{eq:greedy-final-cell}\tag{\theequation.b}\\
  &&&&& \v x\cdot\v x -\textstyle\frac{2}{n}\sum_{j=1}^n\v x\cdot\v
  x_j \leq r^2- \textstyle\frac{1}{n^2} \sum_{j,k=1}^n\v x_j\cdot\v
  x_k.\label{eq:greedy-final-ball}\tag{\theequation.c}
\end{alignat}
Due to the quadratic constraint (\ref{eq:greedy-final-ball}), the
inner optimization task is not as simple as a linear or a quadratic
program. However, several standard optimization packages, e.g., CPLEX
or MOSEK, can handle such so-called quadratically constrained 
linear programs (QCLP) rather efficiently, especially 
when there is only one quadratic constraint. Alternatively,
one can use specialized algorithms for linear programming with a
single quadratic constraint \citep{Pan66,MarSch05} or convert the
quadratic constraint to a second-order cone (SOC) constraint and use
general-purpose conic optimization methods.

\subsubsection{Implementation of a greedy optimal attack}
\label{subsec:implem}

Some additional work is needed for a practical implementation of a
greedy optimal attack against a nearest-out learner.

A point can become ``immune'' to a poisoning attack, if its Voronoi
cell does not overlap with the hypersphere of radius $r$ centered at
$\v c_k$, at some iteration $k$. The quadratic constraint
(\ref{eq:greedy-raw-ball}) is never satisfied in this case, and the
inner optimization problem (\ref{eq:greedy-raw}) becomes
infeasible. From then on, a point remains in the working set forever
and slows down the attack progress. To avoid this awkward situation, an attacker must
keep track of all optimal solutions $\v x_i^{*}$ of the inner
optimization problems.  If any $\v x_i^{*}$ slips out of the
hypersphere after replacing the point $\v x_\alpha$ with $\v
x_\alpha^{*}$, an attacker should ignore the outer loop decision
(\ref{eq:outer-loop}) and instead replace $\v x_i$ with $\v x_i^{*}$.

A significant speedup can be attained by avoiding the solution of
unnecessary QCLP problems. Let $S = \{1, \ldots, i-1\}$ and
$\alpha_S$ be the current best solution of the outer loop problem
  (\ref{eq:outer-loop}) over the set $S$. Let $f_{\alpha_S}$ be the
corresponding objective value of an inner optimization problem
(\ref{eq:greedy-final}). Consider the following auxiliary quadratic
program (QP):
\refstepcounter{equation}
\label{eq:greedy-aux} 
\begin{alignat}{3}
  && \text{max}_{\v x}\;\; &&& \Vert\v x - \textstyle\frac{1}{n}\sum_{j=1}^n\v
  x_j \Vert \tag{\theequation.a}\\
  &&\text{s.t.}\;\; &&& 2(\v x_j-\v x_i)\cdot \v x \leq \v x_j\cdot \v
  x_j - \v x_i\cdot \v x_i , \quad
  \forall j=1,...,n \label{eq:greedy-aux-cell}\tag{\theequation.b}\\
  &&&&& (\v x - \v x_i) \cdot \v a \geq
  f_{\alpha_S}.\label{eq:greedy-aux-obj}\tag{\theequation.c}
\end{alignat}
Its feasible set comprises the Voronoi cell of $\v x_i$, defined by
constraints (\ref{eq:greedy-aux-cell}), further reduced by constraint
(\ref{eq:greedy-aux-obj}) to the points that improve the current value
$f_{\alpha_S}$ of the global objective function. If the objective
function value provided by the solution of the auxiliary QP
(\ref{eq:greedy-aux}) exceeds $r$ then the solution of the local QCLP
(\ref{eq:greedy-final}) does not provide an improvement of the global
objective function $f_{\alpha_S}$. Hence an expensive QCLP
optimization can be skipped.

\subsubsection{Attack Effectiveness}
\label{subsec:greedy-emp}

To evaluate the effectiveness of a greedy optimal attack, we perform a
simulation on an artificial geometric data. The goal of this
simulation is investigate the behavior of the relative displacement
$D_i$ during the progress of a greedy optimal attack. 

An initial working set of size $n = 100$  is sampled from a
$d$-dimensional Gaussian distribution with unit covariance
(experiments are repeated for various values of
$d\in\{2,...,100\}$). The radius $r$ of the online centroid learner is
chosen such that the expected false positive rate is bounded by
$\alpha=0.001$. An attack direction $\v a$, $\Vert \v a \Vert = 1$ is
chosen randomly, and 500 attack iterations ($5 * n$) are generated
using the procedure presented in Sections \ref{subsec:optimal-attack}
-- \ref{subsec:implem}. The relative displacement of the center in the
direction of attack is measured at each iteration. For statistical
significance, the results are averaged over $10$ runs.

Figure~\ref{fig:res-greedy-dim} shows the observed progress of the
greedy optimal attack against the nearest-out learner and compares it
to the behavior of the theoretical bounds for the infinite-horizon
learner (the bound of Nelson et al.) and the average-out learner. The
attack effectiveness is measured for all three cases by the relative
displacement as a function of the number of iterations. Plots for the
nearest-out learner are presented for various dimensions $d$ of the
artificial problems tested in simulations. The following two
observations can be made from the plots provided in
Figure~\ref{fig:res-greedy-displ}:

Firstly, the attack progress, i.e., the functional dependence of the
relative displacement of the greedy optimal attack against the
nearest-out learner with respect to the number of iterations, is
\emph{linear}. Hence, contrary to the initial intuition, the removal
of nearest neighbors to incoming points does not add security
against a poisoning attack.

Secondly, the slope of the linear attack progress \emph{increases with the
dimensionality of the problem}. For low dimensionality, the relative
displacement of the nearest-out learner is comparable, in absolute
terms, with that of the infinite-horizon learner. For high
dimensionality, the nearest-out learner becomes even less secure
than the simple average-out learner.
By increasing the dimensionality beyond $d>n$ the attack effectiveness cannot be increased.
 Mathematical reasons for such
behavior are investigated in Section~\ref{subsec:greedy-theory}.

A further illustration of the behavior of the greedy optimal attack is
given in Figure~\ref{fig:res-greedy-dim}, showing the dependence of
the average attack slope on the dimensionality. One can see that the
attack slope increases logarithmically with the dimensionality and
wanes out to a constant factor after the dimensionality exceeds the
number of training data points. A theoretical explanation of the observed 
experimental results is given in the next section.

% In Fig.~\ref{fig:res-greedy-dim} we compare the displacement curve of the experimental results  for the different feature space dimensionalities  with the average-out and Nelson bounds. Thereby one can see that with the displacement of the nearest-out learner is linear, unlike the logarithmic Nelson bound. Moreover, while the low feature space dimensionality hinders an adversary to quickly displace the learner, high dimensionalities can even lead to a quicker displacement compared to the trivial average-out update rule. To illustrate the sensibility of the nearest-out learner to feature space dimensionality we compared the attack speeds in Plot~\ref{fig:res-greedy-displ}. Thereby we define the attack speed $v$ by formula $v=\frac{nD}{T}$, which is simply the derivate of the displacement by the attack iterations normalized by $n$. 

\begin{figure}
  \subfigure[]{
    \includegraphics[width=0.49\textwidth]{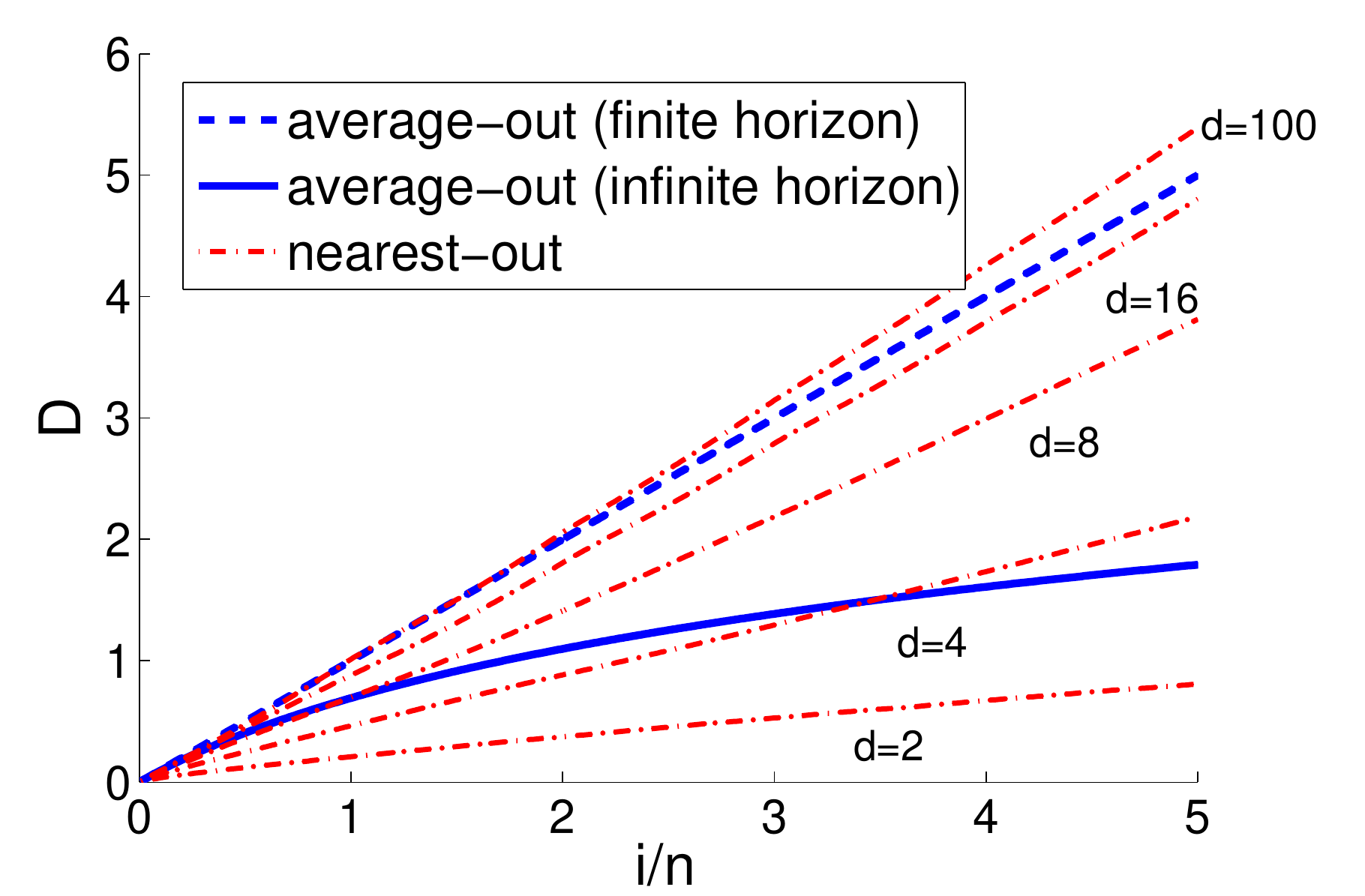}
    \label{fig:res-greedy-displ}
  }
  \subfigure[]{
    \includegraphics[width=0.49\textwidth]{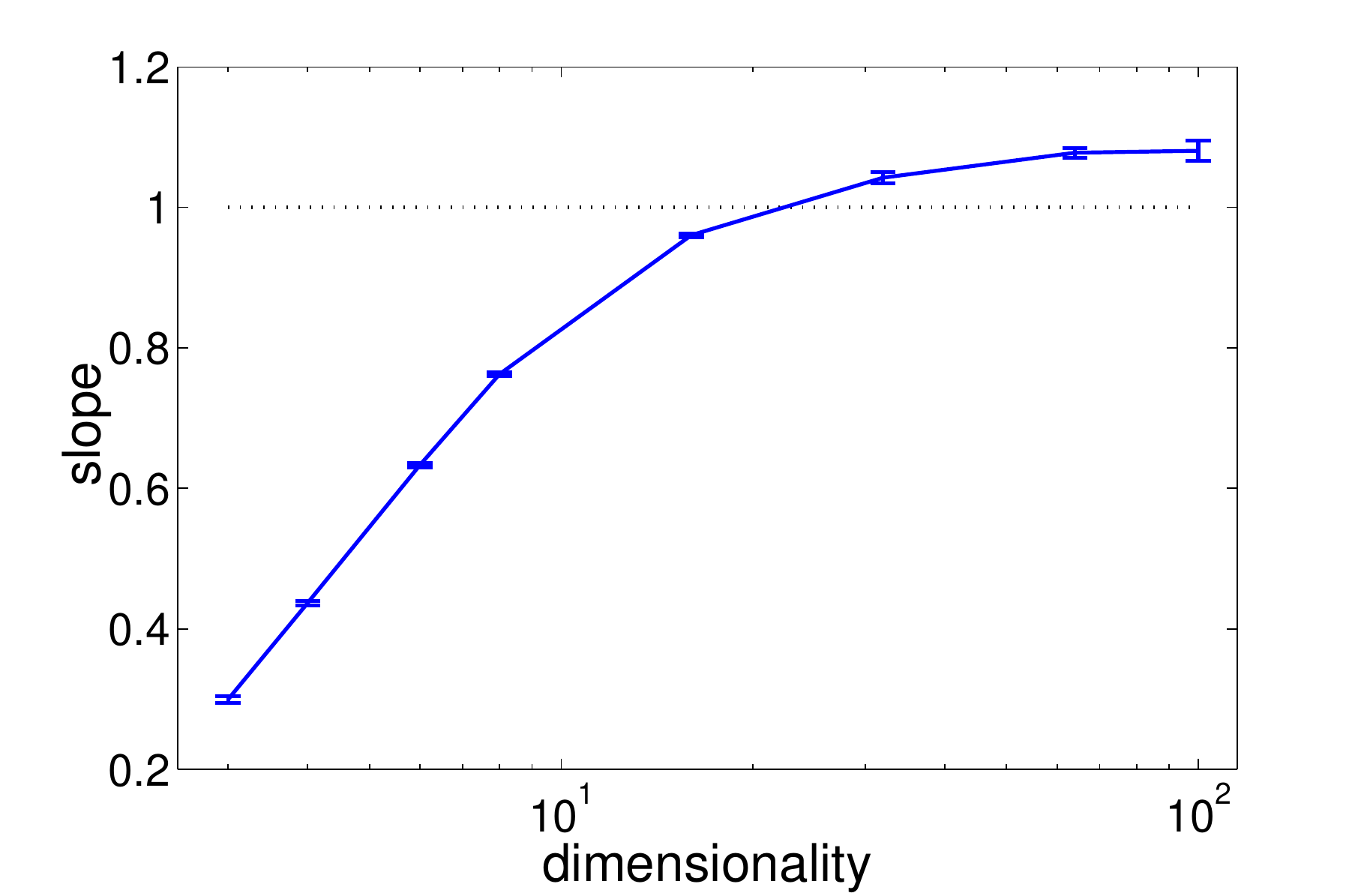}

   \label{fig:res-greedy-dim}
  }
  \caption{Effectiveness of a poisoning attack for the nearest-out
    rule as a function of input space dimensionality.  The
    displacement of a centroid into a direction of an attack grows
    linearly with the number of injected points. The slope of the
    linear growth increases with the input space dimensionality.
    Upper bounds on the displacement of the average-out rule
    rule are plotted for comparison.}
\end{figure}

\subsection{Concluding Remarks}

To summarize our analysis for the case of attacker's full control over
the data, we conclude that an optimal poisoning attack can
successfully subvert a finite-horizon online centroid learner for all
outgoing point selection rules. This conclusion contrasts with the
analysis of the infinite-horizon learner carried out in
\cite{BarNelSeaJosTyg06} that yields a logarithmic attack progress. As
a compromise, one can in practice choose a large working set size $n$,
which reduces the slope of a linear attack progress.

Among the different outgoing point selection rules, the nearest-out
rule presents some challenges to the implementation of an optimal
attack; however, some approximations can make such an attack feasible
while still maintaining a reasonable progress rate. The key factor for
the success of a poisoning attack in the nearest-out case lies in the
high dimensionality of the feature space. The progress of an optimal
poisoning attack depends on the size of Voronoi cells induced by
the training data points. The size of Voronoi cells is related linearly to
the volume of the sphere corresponding to attack's feasible
region. The increasing dimensionality of a feature space blows up the
volume of the sphere and hence causes a higher attack progress rate.

In the following sections we analyze two additional factors that can
affect the progress of a poisoning attack. First, we consider the case
of an attacker being able to control only a fixed fraction $\nu$ of
the training data. Subsequently we analyze a scenario in which an
attacker is not allowed to exceed a certain false positive rate
$\alpha$, e.g., by stopping online learning when a high false positive
rate is observed. In will be shown that both of these possible
constraints significantly reduce the effectiveness of a poisoning
attack.

\section{Poisoning Attack with Limited Bandwidth Constraint}
\label{sec:limited}

We now proceed with investigation of a poisoning attack under a limited bandwidth
constraint imposed on an attacker. We assume that an attacker can
only inject up to a fraction of $\nu$ of the training data. In
security applications, such an assumption is natural, as it may be
difficult for an attacker to surpass a certain amount of innocuous
traffic. For simplicity, we restrict ourselves to the average-out
learner, as we have seen that it only differs by a constant from a
nearest-out one and in expectation equals a random-out one.

%In this section we consider the following probabilistic model to study
%the behavior of a online centroid learner being attacked by an adversary.
%For simplicity we focus on the online centroid learner using with the
%average-out rule. Note that an similar result holds for the random-out
%rule since in average the random-out and the average-out rules
%perform equal updates of the learner's center of mass. It will turn
%out that even this simple online centroid learner until some circumstance may
%be highly secure.

\subsection{Learning and Attack model}
The initial online centroid learner is centered at the position $\v X_0$ and has the
radius $r$ (w.l.o.g. assume $X_0=0$ and $r=1$).  At each iteration a
new training point arrives which is either inserted by an adversary or
is drawn independently from the distribution of innocuous points, and a
new center of mass $\v X_i$ is calculated\footnote{To emphasize the
  probabilistic model used in this section, we denote the location of
  a center and the relative displacement by capital letters.}. The
mixing of innocuous and attack points is modeled by a Bernoulli random
variable with the parameter $\nu$. 
Adversarial points $\v A_i$ are
chosen according to an attack function $f$ depending on the actual state of
the learner $\v X_i$. 
The innocuous pool is modeled by a probability
distribution, from which the innocuous points $\vepsilon_i$ are
independently drawn. We assume that the expectation of
innocuous points $\vepsilon_i$ coincides with the initial center of mass:
$E(\vepsilon_i) = \v X_0$. Furthermore, we assume that all innocuous
points are accepted by the initial learner, i.e., $\Vert \vepsilon_i - X_0\Vert\leq
r$.

%Before the attack the online centroid learner's center of mass is initially located at
%a  position $\v X_0$ having a radius $R$, for the sake of simplicity and without loss of generality $\v X_0=0$ and $R=1$. 
%We index the different positions of the
%learner during the attack via subscripts $i\in\mathbb N_0$, which we
%call attack iterations. In each iteration the center of mass, denoted
%by $\v X_i$, is recalculated because of the arrival of a new training
%point, hence inducing a new center $\v X_{i+1}$. The new point can
%either be inserted by an adversary or being drawn from a pool of
%normal points. Normal and adversarial points are mixed into the
%training data according to a fixed fraction. The mixing is modeled by
%a binary valued random variable $B_i$, which randomly decides whether
%a normal or an adversarial point is presented to the learner. The
%probability that an adversarial point is chosen by $B_i$ is denoted
%with $\nu$, where with probability $1-\nu$ a normal point $\vepsilon_i$
%is drawn from the normal pool. While adversarial points $\v A_i$ are
%chosen according to a clever rule $f$ depending on the actual state of
%the learner $\v X_i$, the normal pool is modeled by a probability
%distribution, from which the normal points $\epsilon_i$ are
%independently drawn. The normal pool's expected value is assumed to
%fall together with the initial center of mass, $E(\vepsilon_i)
%= \v X_0$ and $\Vert \vepsilon_i\Vert\leq r$, i.e. we assume the learner to be in ``perfect condition'' before the attack takes places.

Moreover, for didactical reasons, we make a rather artificial assumption, which we will drop in the next chapter:
\emph{all innocuous points are accepted by the learner, at any time of the attack, independent of their actual distance to the center of mass}.
In the next section we drop this
assumption, such that the learner only accept points which fall within the actual radius.

The described probabilistic model  is formalized by the following axiom.

\begin{axiom}\label{axiom:limited}
  $\{B_i|i\in\mathbb N\}$ are independent Bernoulli random variables
  with parameter $\nu>0$.  $\vepsilon_i$ are i.i.d. random variables
  in a reproducing kernel Hilbert space $\mathcal{H}$, drawn from a fixed but unknown distribution
  $P_\vepsilon$, satisfying $E(\vepsilon_i)= \vzero$ and
  $\Vert\vepsilon_i\Vert\leq r=1$ for each $i$. $B_i$ and
  $\vepsilon_j$ are mutually independent for each $i,j$.  $f:\mathcal H
  \rightarrow\mathcal H$ is an attack strategy satisfying $\Vert
  f(x)-x\Vert\leq r$.  $\{X_i|i\in\mathbb N\}$ is a collection of random
  vectors such that $\v X_0=0$ and
\begin{equation}\label{eq:axiom-limited}
  \v X_{i+1} = \v X_i + \frac{1}{n}\left(B_if(\v X_i) + (1-B_i)
    \vepsilon_i - \v X_i\right).
\end{equation}
\end{axiom}

%Note that the depend random variables $\v X_i$ solely depend on $\{B_j|j<i\}$ and
%$\{\vepsilon_j|j<i\}$, hence $B_i$ and $\vepsilon_i$ are independent of
%$\v X_i$.  

%\begin{definition}
%  For simplicity of notation, we in this section refer to a collection
%  of random vectors $\{\v X_i|i\in\mathbb N\}$ satisfying
%  Axiom~\ref{axiom:limited} as \emph{online centroid learner} denoted by
%  $\mathcal C$. Any function $f$ satifying Ax.~\ref{axiom:limited} is
%  called \emph{attack strategy}.
%\end{definition}

  For simplicity of notation, we in this section refer to a collection
  of random vectors $\{\v X_i|i\in\mathbb N\}$ satisfying
  Axiom~\ref{axiom:limited} as \emph{online centroid learner} denoted by
  $\mathcal C$. Furthermore we denote $\epsilon:=\vepsilon\cdot\v a$. 
  Any function $f$ satisfying Ax.~\ref{axiom:limited} is
  called \emph{attack strategy}. 
  
According to the above axiom an adversary's attack strategy is
formalized by an \emph{arbitrary} function $f$. This raises the
question which attack strategies are optimal in the sense that an
attacker reaches his goal of concealing a predefined attack direction vector in
a minimal number of iterations. An attack's progress is measured by
projecting the current center of mass onto the attack direction vector:

\begin{definition}\label{def:displ}
\hspace{0cm}

 (a)  Let $\v a$ be an attack direction vector (w.l.o.g. $||\v a||=1$), 
  and let $\mathcal C=\{\v X_i|i\in\mathbb N\}$ be a
  online centroid learner. The \emph{$i$-th displacement} of $\mathcal C$,
  denoted by $D_i$, is defined by
$$ D_i=\frac{X_i\cdot\v a}{R} ~ .$$

(b)
	Attack strategies maximizing the 
  displacement $D_i$ in each iteration $i$ are referred to as \emph{optimal attack
    strategies}.
\end{definition}

\subsection{An Optimal Attack}

The following result characterizes an optimal attack strategy for the
model specified in Axiom~\ref{axiom:limited}.

\begin{proposition}\label{prop:limited-optattack}
  Let $\v a$ be an attack direction vector and let $\mathcal C$ be a centroid
  learner. Then the \emph{optimal attack strategy} $f$ is given by
  \begin{equation}\label{eq:limited-optattack}
    f(X_i):=X_i+\v a ~ .
  \end{equation}
\end{proposition}

%For a proof, see Appendix~\ref{app:proofs-limited}.

\begin{proof}
  Since by Axiom~\ref{axiom:limited} we have  $ \Vert
  f(x)-x\Vert \leq r,$ any valid attack strategy can be written as $
  f(x) = x + g(x)$, such that $\Vert g\Vert \leq r = 1.$ It follows
  that
\begin{equation*}
\begin{split}
  D_{i+1} & \leq 
  \v X_{i+1}\cdot\v a \\
  & =  
  \left(\v X_i +\frac{1}{n}\left(B_if(\v X_i) + (1-B_i)\vepsilon_i 
  - \v X_i\right)
  \right) \cdot \v a\\
  &= D_i + \frac{1}{n} ( B_i D_i + B_ig(\v X_i)\!\cdot\! \v a + (1\!-\!B_i) \epsilon_i -
  D_i ) \, .
\end{split}
\end{equation*}
Since $B_i \geq 0$, the optimal attack strategy should maximize $g(\v
X_i)\cdot \v a$ subject to $||g(\v X_i)|| \leq 1$. The maximum is clearly
attained by setting $ g(\v X_i) = \v a$.
\end{proof}

\subsection{Attack Effectiveness}

The estimate of an optimal attack's effectiveness in the limited
control case is given in the following theorem.

%\begin{theorem}\label{th:limited}
%  Let $\mathcal C$ be a online centroid learner satisfying the \emph{optimal} attack
%  strategy. Then for the displacement $D_i$ of $\mathcal C$ we have:
%\begin{eqnarray*}
%  {\rm (a)} ~ ~ ~ & E(D_i) & =  ~ (1-c_i)\frac{\nu}{1-\nu} \\
%  {\rm (b)} ~ ~ ~ & {\rm Var}(D_i) & \leq ~ \gamma_i\left(\frac{\nu}
%  {1-\nu}\right)^2+\delta_n
%\end{eqnarray*}
%where $\gamma_i =c_i-d_i$, $c_i:=\left(1-\frac{1-\nu}{n}\right)^i$, $d_i:=\left(1-\frac{1-\nu}{n}\left(2-\frac{1}{n}\right)\right)^i$ and
%$\delta_n:=\frac{\nu^2 +(1-d_i)}{(2n-1)(1-\nu)^2}$. In particular,
%since $c_i,\delta(n)\rightarrow 0$ for $i,n\rightarrow 0$, we have
%\begin{eqnarray*}
%  {\rm (c)} ~ ~ ~ & E(D_i) & \leq ~ \frac{\nu}{1-\nu} ~ ~ ~ ~ 
%  \text{for all} ~ i \\
%  {\rm (d)} ~ ~ ~ & {\rm Var}(D_i) & \rightarrow ~ 0 ~ ~ ~ ~ ~ 
%  ~ ~ ~ ~ ~ ~ \text{for} ~ i,n\rightarrow\infty ~ .
%\end{eqnarray*}
%\end{theorem}

\begin{theorem}\label{th:limited}
  Let $\mathcal C$ be a centroid learner under an optimal poisoning
  attack. Then, for the displacement $D_i$ of $\mathcal C$, it holds:
\begin{eqnarray*}
  {\rm (a)} ~ ~ ~ & E(D_i) & =  ~ (1-c_i)\frac{\nu}{1-\nu} \\
  {\rm (b)} ~ ~ ~ & {\rm Var}(D_i) & \leq ~ \gamma_i\left(\frac{\nu}
  {1-\nu}\right)^2+\delta_n
\end{eqnarray*}
where $\gamma_i =c_i-d_i$, $c_i:=\left(1-\frac{1-\nu}{n}\right)^i$, $d_i=\left(1-\frac{1-\nu}{n}\left(2-\frac{1}{n}\right)\right)^i$ and
$\delta_n:=\frac{\nu^2 +(1-d_i)}{(2n-1)(1-\nu)^2}$.
\end{theorem}

%For a proof, see Appendix~\ref{app:proofs-limited}.

\begin{proof}
(a) Inserting the optimal attack strategy of Eq.~\eqref{eq:limited-optattack} into
Eq.~\eqref{eq:axiom-limited} of Ax.~\ref{axiom:limited}, we have:
$$ \v X_{i+1} = \v X_i + \frac{1}{n}\left(B_i\left(\v X_i+\v a\right) 
  + (1-B_i)\vepsilon_i - \v X_i\right) ~ , $$
which can be rewritten as:
\begin{equation}\label{eq:lim-main-th}
  \v X_{i+1} = \left(1-\frac{1-B_i}{n}\right)\v X_i + \frac{B_i}{n}\v a 
  + \frac{(1-B_i)}{n}\vepsilon_i ~ .
\end{equation}
Taking the expectation on the latter equation, and noting that by
Axiom~\ref{axiom:limited} $E(\vepsilon)=0$ and $E(B_i)=\nu$ holds, we have
$$  E\left(\v X_{i+1}\right) = \left(1-\frac{1-\nu}{n}\right)E(\v X_i) 
  + \frac{\nu}{n}\v a ~ ,$$
which by Def.~\ref{def:displ} translates to
$$ E(D_{i+1})=\left(1-\frac{1-\nu}{n}\right)E(D_i) + \frac{\nu}{n} ~ .$$
The statement (a) follows from the latter recursive eequation by Prop.~\ref{prop:geom-series} (formula of the geometric series).
For the more demanding proof of (b), see Appendix~\ref{app:proofs-limited}.
\end{proof}

The following corollary shows the asymptotic behavior of the above
theorem. 

\begin{corollary}\label{cor:limited}
  Let $\mathcal C$ be a centroid learner satisfying under an optimal
  poisoning attack. Then. for the displacement $D_i$ of $\mathcal C$,
  it holds:
\begin{eqnarray*}
   {\rm (a)} \quad\quad E(D_i) & \leq & \frac{\nu}{1-\nu} 
     ~ ~\text{for all} ~ i \\
   {\rm (b)} \hspace{2pt}\quad {\rm Var}(D_i) & \rightarrow & 0  \quad \quad
     ~ ~\text{for} ~ i,n\rightarrow\infty .
\end{eqnarray*}
\end{corollary}

\begin{proof}
The corollary follows by ~ $\gamma_i, \delta_n\rightarrow
0$ ~ for $i,n\rightarrow\infty$.
\end{proof}

\begin{figure}
  \centering
  \includegraphics[width=0.57\textwidth]{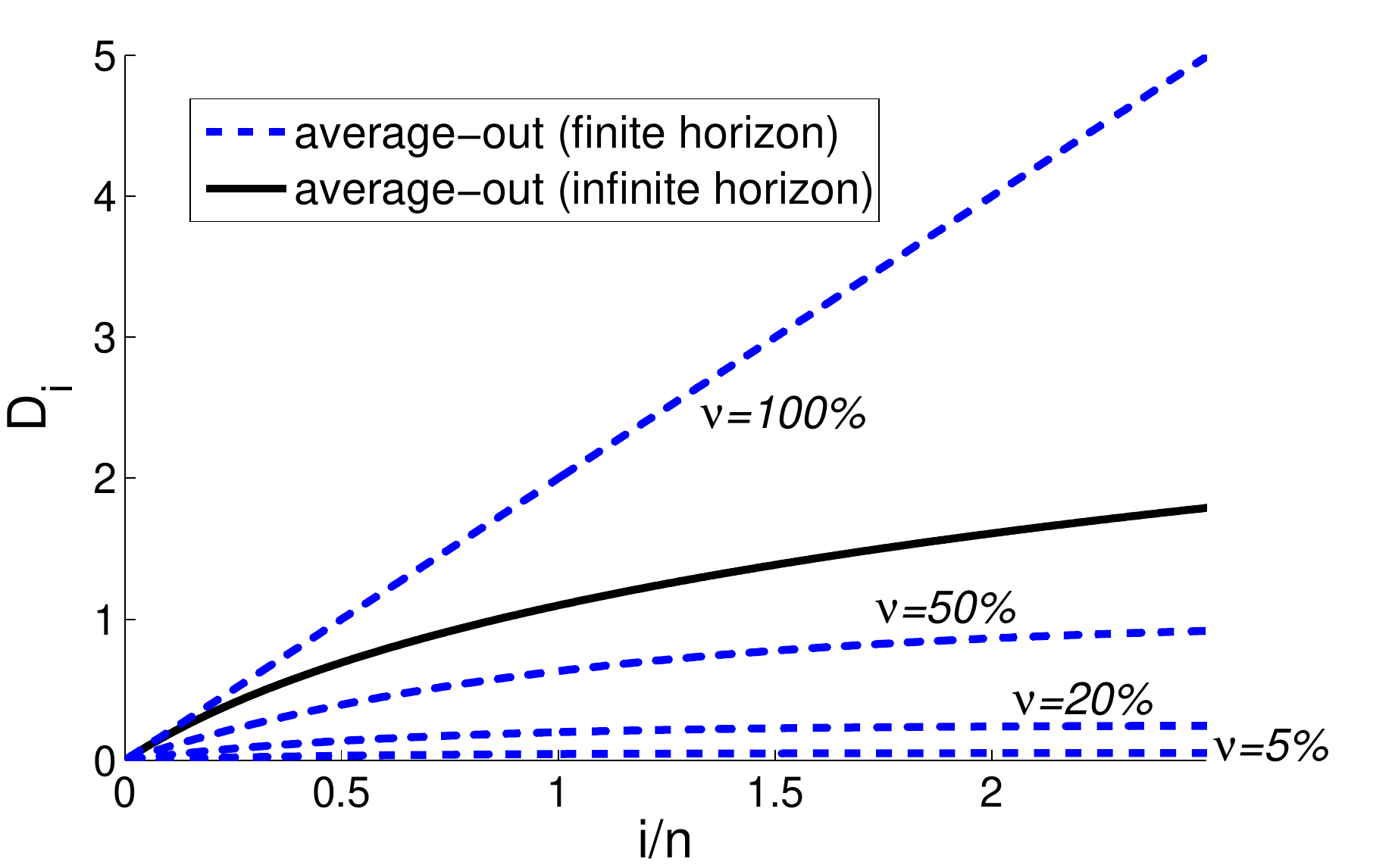}
  \label{fig:bound_limited} 
  \caption{Theoretical behavior of the displacement of a centroid
    under a poisoning attack for a bounded fraction of traffic under
    attacker's control. The infinite horizon bound of Nelson et al. is
    shown for comparison (solid line).}
\end{figure}

The growth of the above bounds as a function of an number of attack
iterations is illustrated in Fig.~\ref{fig:bound_limited}. One can see
that the attack's success strongly depends on the fraction of the
training data controlled by an attacker. For small $\nu$, the attack
progress is \emph{bounded by a constant}, which implies that an attack
fails even with an infinite effort. This result provides a \emph{much
  stronger security guarantee} than the exponential bound for the
infinite horizon case.

\begin{figure}
  \centering
  \includegraphics[width=0.57\textwidth]{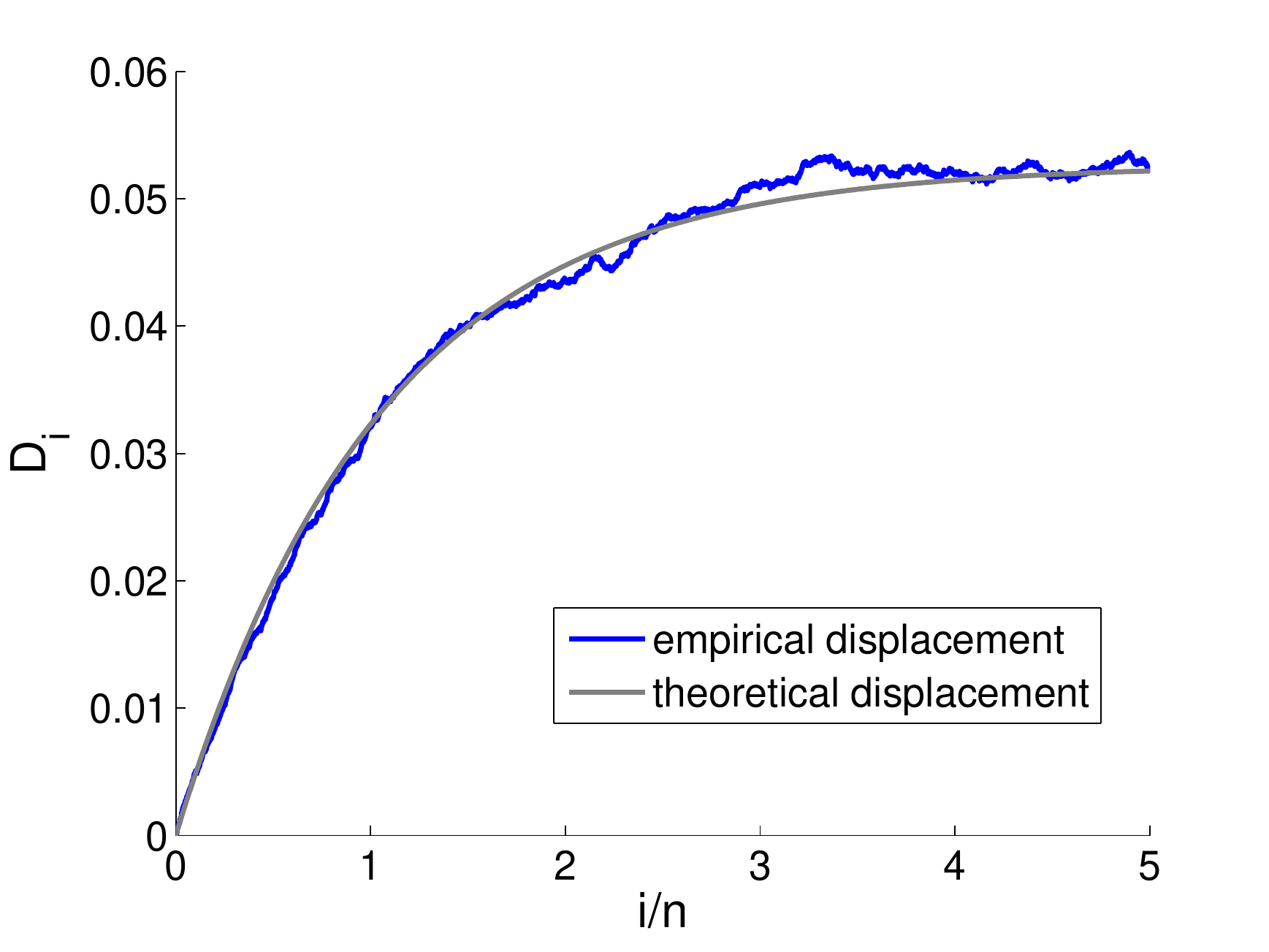}
  \label{fig:montecarlo_axiom6}
  \caption{Comparison of empirical displacemend of the centroid under
    poisoning attack with attacker's limited control ($\nu=0.05$) with
    a theoretical bound for the same setup. Emprical results are
    averaged over 10 runs; standard deviation is shown by vertical
    bars.}
\end{figure}

To empirically investigate the tightness of the derived bound we compute a Monte Carlo simulation of Axiom 6 with the
parameters $\nu=0.05$, $n=100000$, $\mathcal H=\mathbb R^2$, and $\epsilon$ being a uniform distribution over the unit circle.
Fig.~\ref{fig:montecarlo_axiom6} shows a typical displacement curve over the first $500,000$ attack iterations. 
Errorbars are computed over $10$ repetitions of the simulation.

\section{Poisoning Attack under False Positive Constraints}\label{sec:defense}

In the last section we have assumed, that innocuous training points
$\vepsilon_i$ are always accepted by the online centroid learner. But while an
attacker displaces the hypersphere, it may happen that some innocuous
points drop out of the hypersphere's boundary. 
We have seen that an
attacker's impact highly depends on the fraction of points he places.
If an attacker succeeds in pushing the hypersphere far enough such that
sufficiently many innocuous points drop out, he can quickly displace the
hypersphere.

%\footnote{Indeed it is even possible to construct
%distributions of the training data $P_X$, such that for every
%$\nu>0$ the bound of Theorem~\ref{th:limited}(b) drops to infinity
%and hence is useless. E.g. consider the distribution on $\mathbb R$
%defined by $P(\epsilon_i=1)=\frac{1}{2}$ and $P(\epsilon=-1)=\frac{1}{2}$.}
%Hence we have make sure, that only a small fraction of normal training
%points drop out.

\subsection{Learning and Attack Model}

Motivated by the above considerations we modify the probabilistic model of the last section as
follows.
%\footnote{We also calculated bounds only depending on $n$,
%$\nu$ and - instead of $\alpha$ - $F_X$, the cumulative distribution
%function of $P_X$. Unfortunately it is possible to show that for
%every $\nu>0$ there exist distributions of training data, such that
%for every $\nu>0$ an attacker's progress under the optimal attack
%strategy is linear in number of attack iterations. Thus it is not
%possible to derive bounds in the style of Theorem~\ref{th:limited}(b)
%which don't drop to infinity.} 
Again we consider a online centroid learner initially anchored at a
position $\v X_0$ having a radius $r$, for the sake of simplicity and without loss of generality $\v X_0=0$ and $r=1$.
Then innocuous and adversarial
points are mixed into the training data according to a fixed fraction,
controlled by a binary valued random variable $B_i$. But now, in contrast to the last section, innocuous
points $\vepsilon_i$ are only accepted if and only if they fall
within a radius of $r$ of the hypersphere's center $\v X_i$. In addition, to avoid the learner
being quickly displaced, we require that the false alarm rate
is bounded by $\alpha$. If the latter is exceeded, we assume the
adversary's attack to have failed, i.e., a safe state of the learner is
loaded and the online update mechanism is temporarily switched off.

We formalize the probabilistic model as follows:

\begin{axiom}\label{axiom:defense}
  $\{B_i|i\in\mathbb N\}$ are independent Bernoulli random variables
  with parameter $\nu>0$.  $\vepsilon_i$ are i.i.d. random variables
  in a reproducing kernel Hilbert space $\mathcal{H}$, drawn from a fixed but unknown distribution
  $P_\vepsilon=P_{-\vepsilon}$, satisfying $E(\vepsilon_i)= \vzero$, and
  $\Vert\vepsilon_i\Vert\leq r=1$ for each $i$. $B_i$ and
  $\vepsilon_j$ are mutually independent for each $i,j$.  $f:\mathcal
  H\rightarrow\mathcal H$ is an attack strategy satisfying $\Vert
  f(x)-x\Vert\leq r$.  $\{X_i|i\in\mathbb N\}$ is a collection of random
  vectors such that $\v X_0=\v 0$ and
  \begin{eqnarray}\label{eq:axiom-defense}
    \v X_{i+1} = \v X_i + \frac{1}{n}\left(B_i\left(f(\v X_i)-\v X_i\right) 
    + (1-B_i)I_{\{\Vert \vepsilon_i-\v X_i\Vert\leq r\}}\left(\vepsilon_i 
    - \v X_i\right)\right) ~ ,
 \end{eqnarray}
 if $E_{\vepsilon_i}\left(I_{\{\Vert \vepsilon_i-\v X_i\Vert\leq
    r\}}\right)\leq 1-\alpha$ and by 
%$\v X_{i+1}=-\infty \cdot \v a$
$\v X_{i+1}=\vzero$
elsewise. 
\end{axiom}

%Since $X_i$ only randomly depends
%on $\{B_j|j<i\}$ and $\{\vepsilon_j|j<i\}$, $B_i$ and
%$\vepsilon_i$ are independent of $\v X_i$.

%We denote $\epsilon:=\vepsilon\cdot\v a$.

  For simplicity of notation, we in this section refer to a
	collection of random vectors $\{\v X_i|i\in\mathbb N\}$ satisfying
  Ax.~\ref{axiom:defense} as \emph{online centroid learner with maximal false positive rate $\alpha$} denoted by
  $\mathcal C$. Any function $f$ satisfying Ax.~\ref{axiom:defense} is
  called \emph{attack strategy}. Optimal attack strategies are characterized in term of the displacement as in the previous section (see Def.~\ref{def:displ}).

\subsection{Optimal Attack and Attack Effectiveness}

The following result characterizes an optimal attack strategy for the
model specified in Axiom~\ref{axiom:defense}.
%For technical reason in the setting considered here it is in general not possible to 
%construct an optimal attack in the style of Prop.~\ref{prop:limited-optattack}. Instead we derive a \textit{bound} on the impact of
%  an optimal attack, which is done below.

%\begin{proposition}\label{prop:defense-optattack}
%Let $\mathcal C$ be a protected online centroid learner satisfying the optimal 
%attack strategy. Then we have:
%$$  D_i\leq\left(1-\frac{1-B_i}{n}\right)D_i + \frac{B_i}{n}
%  + \frac{1-B_i}{n}\left(I_{\{\Vert \vepsilon_i - \v X_i \Vert >r\}}D_i
%  + I_{\{\Vert \vepsilon_i - \v X_i \Vert\leq r\}}\epsilon_i\right) ~ ,$$
%denoting $\epsilon_i = \vepsilon_i\cdot \v a$.
%\end{proposition}

\begin{proposition}\label{prop:defense-optattack}
  Let $\v a$ be an attack direction vector and let $\mathcal C$ be a centroid learner
  with maximal false positive rate $\alpha$.
  Then an \emph{optimal} attack strategy $f$ is given by
$$
    f(X_i):=X_i+\v a ~ .
$$
\end{proposition}

\begin{proof}
%The proof is similar to the one of Prop.~\ref{prop:limited-optattack} from the previous section and is given in App.~\ref{app:defense}.
  Since by Axiom~\ref{axiom:defense} we have  $ \Vert
  f(x)-x\Vert \leq r,$ any valid attack strategy can be written as $
  f(x) = x + g(x)$, such that $\Vert g\Vert \leq r = 1.$ It follows
  that either $D_i=0$, in which case the optimal $f$ is arbitrary, or we have
\begin{eqnarray*}
  D_{i+1} & = & 
  \v X_{i+1}\cdot\v a \\
  & = &  
  \left(\v X_i +\frac{1}{n}\left(B_if(\v X_i) + (1-B_i)\vepsilon_i 
  - \v X_i\right)
  \right) \cdot \v a\\
  &= &  D_i 
  +\frac{1}{n}\left(B_i\left(D_i+g(\v X_i)\right) 
  + (1-B_i)\vepsilon_i - D_i\right) 
\end{eqnarray*}
Since $B_i \geq 0$, the optimal attack strategy should maximize $g(\v
X_i)\cdot \v a$ subject to $||g(\v X_i)|| \leq 1$. The maximum is clearly
attained by setting $ g(\v X_i) = \v a$.
%The optimal attack strategy $f(x)=x+g(x)$ is by definition a maximizer of above term. 
%Since the right summand does not depend on $g$, this is equivalent to maximizing the left term, i.e. 
%$\frac{1}{n}B_ig(\v X_i)$.
%Maximizing the
%latter with respect to $g$ corresponds to maximizing $g(X_i)\cdot\v
%a$ (note $B_i\geq 0$), hence by the  Cauchy-Schwarz inequality we have the maximizer 
%$g(X_i) = \lambda\v a$, furthermore by \eqref{eq:attack-decomp-2} $\lambda=1$, thus $f(x)=x+\v a$.
\end{proof}

The estimate of an optimal attack's effectiveness in the limited
control case is given in the following main theorem of this paper.

\begin{theorem}\label{th:limited-main}
 Let $\mathcal C$ be a centroid learner with maximal false positive rate $\alpha$  
 under a poisoning attack. 
 Then, for the displacement $D_i$ of $\mathcal C$, it holds:
\begin{eqnarray*}
  {\rm (a)} ~ ~ ~ &E(D_i)&\leq (1-c_i)\frac{\nu+\alpha(1-\nu)}{(1-\nu)
  (1-\alpha)} \\
  {\rm (b)} ~ ~ ~ &{\rm Var}(D_i)&\leq \gamma_i\frac{\nu^2}{(1-\alpha)^2(1-\nu)^2}  + \rho(\alpha) + \delta_n
\end{eqnarray*}
where $c_i:=\left(1-\frac{(1-\nu)(1-\alpha)}{n}\right)^i$, ~
$d_i=\left(1-\frac{1-\nu}{n}(2-\frac{1}{n})(1-\alpha)\right)^i$, ~
$\gamma_i=(c_i-d_i)$, ~
$\rho(\alpha)=\alpha\frac{(1-c_i)(1-d_i)(2\nu(1-\alpha)+\alpha)}{(1-\frac{1}{2n})(1-\nu)^2(1-\alpha)^2}$, ~ 
and $\delta_n=\frac{(1-d_i)(\nu+(1-\nu)E(\epsilon_i^2))}{(2n-1)(1-\nu)(1-\alpha)}$.
\end{theorem}

The proof is technically demanding and is given in App.~\ref{app:defense}. 
Despite the more general proof reasoning, we recover the tightness of the bounds of the previous section
for the special case of $\alpha=0$, as shown by the following corollary.

\begin{corollary}\label{cor:tightness}
  Suppose a maximal false positive rate of $\alpha=0$. 
  Then, the bounds on the expected displacement $D_i$, as given by Th.~\ref{th:limited}~and~Th.~\ref{th:limited-main}, coincident. Furthermore,
  the variance bound of Th.~\ref{th:limited-main} upper bounds the one of Th.~\ref{th:limited}.
\end{corollary}

\begin{proof}
We start by setting $\alpha=0$ in Th.~\ref{th:limited-main}(a). Then, clearly the latter bound coincidents with its counterpart in Th.~\ref{th:limited}.
For the proof of the second part of the corollary, we observe that $\rho(\alpha)=0$ and that the quantities $c_i,d_i,$ and $\gamma_i$ coincident with its counterparts in Th.~\ref{th:limited}. Moreover, removing the distribution dependence by upper bounding $E(\epsilon_i)\leq 1$ reveals that $\delta_i$ is upper bounded by its counter part of Th.~\ref{th:limited}. Hence, the whole expression on the right hand side of Th.~\ref{th:limited-main}(b) is upper bounded by its counterpart in Th.~\ref{th:limited}(b).
\end{proof}

The following corollary shows the asymptotic behavior of the above theorem. It follows from 
 $\gamma_i, \delta_n,\rho(\alpha)\rightarrow 0$ for
$i,n\rightarrow\infty$,  and $\alpha\rightarrow 0$, respectively.
\begin{corollary}
Let $\mathcal C$ be a centroid learner
  with maximal false positive rate $\alpha$ satisfying the
optimal attack strategy. Then for the displacement of $\mathcal C$,
denoted by $D_i$, we have:
\begin{eqnarray*}
  {\rm (a)} ~ ~ ~ & E(D_i) &   \leq ~ ~ \frac{\nu+\alpha(1-\nu)}
  {(1-\nu)(1-\alpha)} ~ ~ ~ ~ ~ ~  ~ ~ ~ ~ \text{for all} ~ i 
  \hspace{3.8cm}\\
  {\rm (b)} ~ ~ ~ & {\rm Var}(D_i) &   \rightarrow ~ 0  
   \quad\quad\quad\quad\quad\quad\quad\quad\quad ~ \text{for} ~ i,n\rightarrow\infty, \alpha\rightarrow 0 ~ .
\end{eqnarray*}
\end{corollary}

From the previous theorem, we can see that for small false positive rates $\alpha\approx 0$, which are common in many applications, e.g., Intrusion Detection (see Sect.~\ref{sec:ids} for an extensive analysis),
the bound approximately equals the one of the previous section, i.e., we have
$E(D_i)\leq\frac{\nu}{1-\nu}+\delta$ where $\delta> 0$ is a small constant with $\delta\rightarrow 0$. 
Inverting the bound we obtain the useful formula
\begin{equation}\label{eq:nu-crit}
 \nu\geq\frac{E(D_i)}{1+E(D_i)}
\end{equation}
which gives a lower bound on the minimal $\nu$ an adversary has to employ for an attack
to succeed.

The bound of Th.~\ref{th:limited-main} is shown in 
Fig.~\ref{fig:bound_limited} for different levels
of false positive protection $\alpha\in[0,0.025]$. 
We are especially interested in low positive rates
which are common in anomaly detection applications.
One can see that much of the 
tightness of the bounds of the previous section is preserved.
In the extreme case $\alpha=0$ the bounds coincident, as been shown in Cor.~\ref{cor:tightness}.

\begin{figure}
  \begin{center}
  \includegraphics[width=0.65\textwidth]{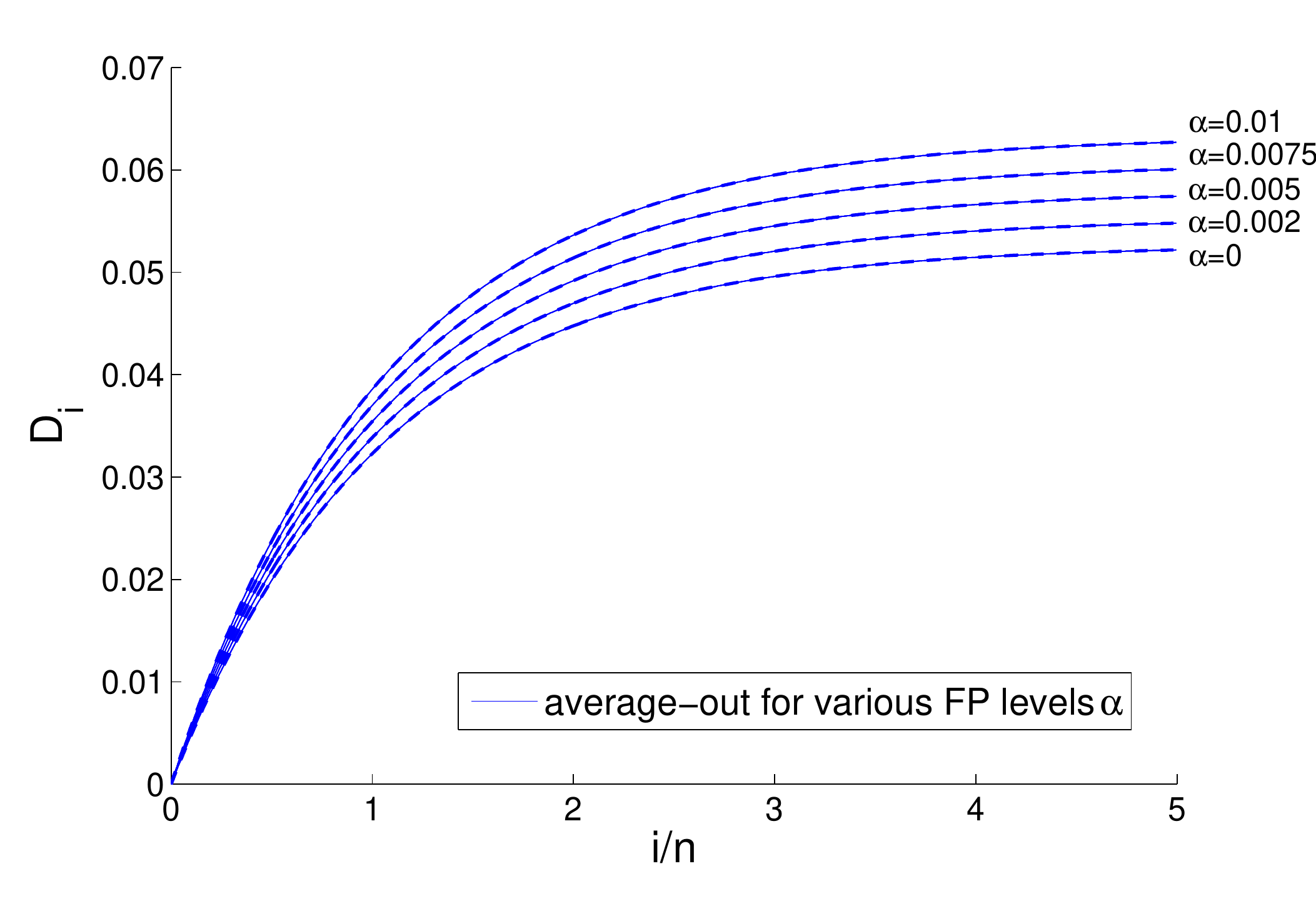}
  \end{center}
  \label{fig:bound-defense} 
  \caption{Theoretical behavior of the displacement of a centroid
    under a poisoning attack for different levels of false positive
    protection $\alpha$. The predicted displacement curve for
    $\alpha=0$ coincides with the one shown in
    Fig.~\ref{fig:montecarlo_axiom6}.}
\end{figure}

\section{Case Study: Application to Intrusion Detection}\label{sec:ids}

In this section we present the experimental evaluation of the
developed analytical instruments in the context of a particular
computer security application: intrusion detection. Centroid anomaly
detection has been previously used in several intrusion detection
systems
\cite[e.g.,][]{HofForSom98,LazErtKumOzgSri03,WanSto04,LasSchKotMue04,WanCreSto05,RieLas06,WanParSto06,RieLas07}. After
a short presentation of data collection, preprocessing and model
selection, 
our experiments aim at verification of the theoretically obtained
growth rates for attack progress as well as computation of constant
factors for specific exploits.
%our experiments are organized according to the two main
%cases considered in the analysis: full and partial control by an
%attacker.

\subsection{Data Corpus and Preprocessing}\label{subsec:corpus}

The data to be used in our case study represents real HTTP traffic
recorded at Fraunhofer FIRST. We consider the intermediate granularity
level of requests which are the basic application-layer syntactic
elements of the HTTP protocol. Packet headers have been stripped, and
requests spread across multiple packets have been merged together. The
resulting benign dataset consists of 2950 byte strings containing
payloads of inbound HTTP requests. The malicious dataset consists of
69 attack instances from 20 classes generated using the Metasploit
penetration testing framework\footnote{{http://www.metasploit.com/}}. 
All exploits were normalized to match the frequent 
attributes of innocuous HTTP requests such that the malicious payload provides 
the only indicator for identifying the attacks. 

As byte sequences are not directly suitable for application of machine
learning algorithms, 
we deploy a $k$-gram spectrum kernel
\citep{LesEskNob02,ShaCri04} for the computation of the inner
products. To enable fast comparison of large byte
sequences (a typical sequence length 500-1000 bytes),
efficient algorithms using sorted arrays \citep{RieLas08} have been
implemented. Furthermore, kernel values are normalized according to
\begin{equation} 
  \label{eq:norm} 
  k(\v x,\bar{\v x}) 
  \longmapsto 
  \frac{k(\v x,\bar{\v x})}
  {\sqrt{ k(\v x,\v x) k(\bar{\v x}, \bar{\v x})}} ~ ,
\end{equation}
to avoid a dependence on the length of a request payload.
The resulting inner products subsequently have been processed by an
RBF kernel.

\subsection{Learning Model}

The feature space selected for our experiments depends on two
parameters: the $k$-gram length and the RBF kernel width
$\sigma$. Prior to the main experiments aimed at the validation of
proposed security analysis techniques, we investigate optimal model
parameters in our feature space. The parameter range considered is 
$k=1,2,3$ and $\sigma=2^{-5},2^{-4},...,2^{5}$.

To carry out model selection, we randomly partitioned the innocuous
corpus into disjoint training, validation and test sets (of sizes
$1000$, $500$ and $500$). The training partition is comprised of the innocuous data
only, as the online centroid learner assumes clean training data. The
validation and test partitions are mixed with $10$ attack instances
randomly chosen from \emph{different} attack classes.\footnote{The
  latter requirement reflects the goal of anomaly detection to
  recognize \emph{previously unknown} attacks.} For each partition,
different online centroid learner models are trained on a training set and
evaluated on a validation and a test sets using the 
normalized\footnote{such that an AUC of $1$ is the highest achievable value} AUC$_{[0,0.01]}$ as a
performance measure. For statistical significance, model selection is
repeated $1000$ times with different randomly drawn partitions. The
average values of the normalized AUC$_{[0,0.01]}$ for the different
$k$ values on test partitions are given in Table~\ref{tab:accur}.

It can be seen that the $3$-gram model consistently shows better AUC
values for both the linear and the best RBF kernels. We have chosen
the linear kernel for the remaining experiments, since it allows to
carry out computations directly in input space with only a marginal
penalty in detection accuracy. 

% The linear $3$-gram model has been determined to have a
% AUC$_{[0,0.01]}$ value of $0.989$ by a standard deviation of $0.018$.
% Furthermore we still have a detection rate of over $95\%$ by no false
% positives. Some RBF models turned out to be competitive to the linear model but not significantly better. Therefore - for
% simplicity of analysis - we preferred to take the linear model for the security analysis of the consequent sections.

%\begin{figure}
%\includegraphics[width=0.75\textwidth]{roc}
%\caption{Experimental results of intrusion detection experiments for the 
%online centroid learner.}
%\label{fig:ids-model}
%\end{figure}

\addtolength{\extrarowheight}{2pt}
\begin{table}[h]
\centering
\begin{tabular}{lccc}
  & linear & \text{best RBF kernel} &  optimal $\sigma$  \\
  \text{1-grams} & $0.913\pm 0.051$ & $0.985\pm 0.021$ & $2^{-2.5}$ \\
  \text{2-grams} & $0.979\pm 0.026$ & $0.985\pm 0.025$ & $2^{-1.5}$ \\
  \text{3-grams} & $0.987\pm 0.018$ & $0.989\pm 0.017$ & $2^{-0.5}$
\end{tabular} 
\caption{Accuracy of the linear kernel and the best RBF kernel as well as the optimal bandwidth $\sigma$.}
\label{tab:accur}
\end{table}

\subsection{Intrinsic HTTP Data Dimensionality}\label{subsec:dim}

Dimensionality of training data makes an important contribution to
the (in)security of the online centroid learner when using the nearest-out update rule.
Simulations on artificial
data (cf. Section~\ref{subsec:greedy-emp}) show that the slope of a
linear progress rate of a poisoning attack increases for larger
dimensionalities $d$. This can be also explained theoretically
(cf. Section~\ref{subsec:greedy-theory}) by the fact that radius of
Voronoi cells induced by training data is proportional to
$\sqrt[d]{1/n}$, which increases with growing $d$.

For the intrusion detection application at hand, the dimensionality of
the chosen feature space ($k$-grams with $k = 3$) is $256^3$. In view
of Th.~\ref{th:span}, the dimensionality of the relevant
subspace in which attack takes place is bounded by the size of
the training data $n$, which is much smaller, in the range of 100 -- 1000
for realistic applications. Yet the real progress rate depends on the
\emph{intrinsic} dimensionality of the data. When the latter is
smaller than the size of the training data, an attacker can compute a
PCA of the data matrix \citep{SchSmoMue98} and project the original
data into a subspace spanned by a smaller number of informative
components. 

To determine the intrinsic dimensionality of possible training sets
drawn from HTTP traffic, we randomly drew $1000$ elements from the
training set, calculate a linear kernel matrix in the space of
$3$-grams and compute its eigenvalue decomposition. We then determine
the number of leading eigen-components preserving as a function of the
percentage of variance preserved. The results averaged over 100
repetitions are shown in Fig.~\ref{fig:ids-dim}.

\begin{figure}[tbhp]
\begin{center}
\includegraphics[width=0.65\textwidth]{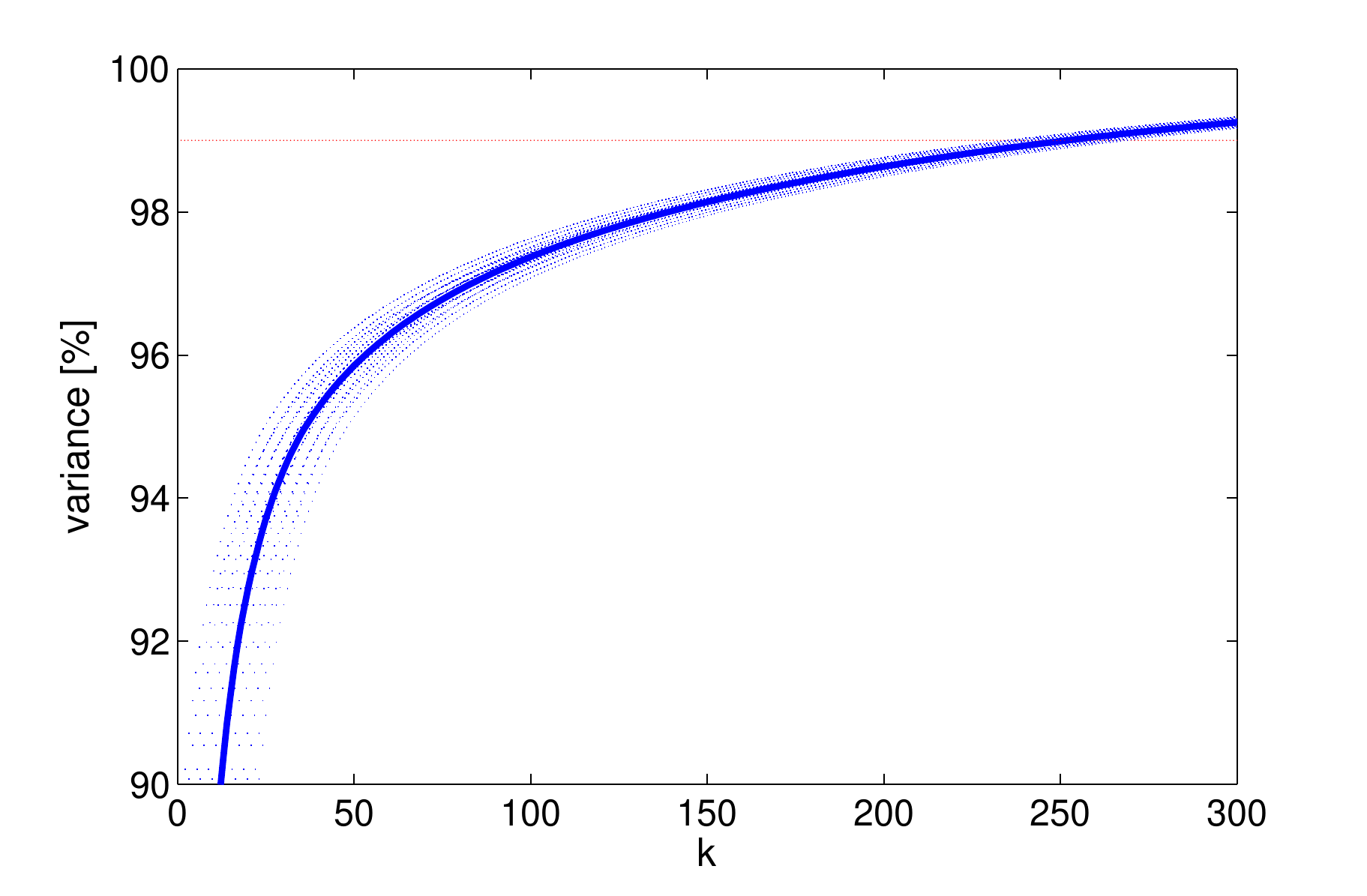}
\end{center}
\caption{Intrinsic dimensionality of the embedded HTTP data. The
  preserved variance is plotted as a function of the number of
  eigencomponents, $k$, employed for calculation of variance (solid
  blue line). The tube indicates standard deviations.}
\label{fig:ids-dim}
\end{figure}
It can be seen that $250$ kernel PCA components are needed to preserve
99\% of the variance. This implies that, although effective
dimensionality of HTTP traffic is significantly smaller that the
number of training data points, it still remains sufficiently high so
that the rate of attack progress approaches 1, which is similar to the
simple average-out learner.

\subsection{Geometrical Constraints of HTTP Data}
\label{subsec:ids-greedy}

Several technical difficulties arising from  data geometry have to be overcome in launching a
poisoning attack in practice. It turns out, however, that the consideration of the training data
geometry provides an attacker with efficient
tools for finding reasonable approximations for the above mentioned
tasks.

(1) First, we cannot directly simulate a 
poisoning attack in the $3$-gram input space due to
its high dimensionality. 
An approximately equivalent explicit feature space can be
constructed by applying kernel PCA to the kernel matrix $K$. 
By pruning the eigenvalues ``responsible'' for dimensions with low
variance one can reduce the size of the feature space to the implicit
dimensionality of a problem if the kernel matches the
data \citep{BraBuhMue08}. In all subsequent experiments we used $d
= 256$ as suggested by the experiments in Section~\ref{subsec:dim}.

(2) Second the crucial
normalization condition~(\ref{eq:norm}) 
requires that a solution lies on a unit sphere.\footnote{In the absence
of normalization, the high variability of the byte sequence
lengths leads to poor accuracy of the centroid anomaly detection.}
Unfortunately, this renders the calculation of an optimal attack point
non-convex. Therefore we
pursue the following heuristic procedure to enforce normalization:
we explicitly project local solutions (for each Voronoi
cell) to a unit sphere, verify their feasibility (the radius and the
cell constraints), and remove infeasible points from the outer
loop~(\ref{eq:outer-loop}).
%\begin{itemize}
%\item Compute a greedy optimal attack for each Voronoi cell by solving
%  the original optimization problem~(\ref{eq:greedy-final}).
%\item Project an optiomal solution onto the unit sphere.
%\item Verify feasibility constraints and remove infeasible points.
%\item Select the best point among the remaining locally optimal
%  solutions.
%\end{itemize}

(3) In general one cannot expect each feature space vector to correspond to
a valid byte sequence since not all combinations of $k$-grams can be
``glued'' to a valid byte sequence. In fact, finding a sequence with the
best approximation to a given $k$-gram feature vector has been shown
to be NP-hard \citep{FogLee06}. Fortunately
by the fact that an optimal
attack lies in the span of training data, i.e. Th.~\ref{th:span}, 
we construct an attack's
byte sequence by concatenating original sequences of basis points
with rational coefficients that approximately match the coefficients
of the linear combination.  A potential
disadvantage of this method is the large increase in the sequence lengths.
Large requests are conspicuous and may consume significant
resources on the attacker's part.

(4) An attack byte sequence must be embedded in a valid HTML
protocol frame.
Building a valid HTTP request with arbitrary content is, in
general, a non-trivial task, especially if it is required that a
request does not cause an error on a server. An HTTP request consists
of fixed format headers and a variable format body. A most
straightforward way to stealthily introduce arbitrary content is to
provide a body in a request whose method (e.g., GET) does not require
one. According to an RFC specification of the HTTP protocol, a request
body should be ignored by a server in this case.

\subsection{Poisoning Attack for Finite Horizon Centroid Learner}

The analysis carried out in Section~\ref{sec:full} shows that an
online centroid learner, in general, does not provide sufficient
security if an attacker fully controls the data. Practical efficiency
of a poisoning attack, however, depends on the dimensionality and geometry of training
data analyzed in the previous section. 
Theoretical results have been illustrated in simulations on
artificial data presented in
Section~\ref{subsec:greedy-emp}. Experiments in this section are
intended to verify whether these findings hold for real attacks
against HTTP applications. 
Our experiments focus on the nearest-out learner, as other update
rules can be easily attacked with trivial methods. 

We are now in the position to evaluate the progress rate of a
poisoning attack on real network traffic and exploits. The goal of
these experiments is to verify simulations carried out in Section
\ref{subsec:implem} on real data. 

Our experimental protocol is as follows. We randomly draw $n=250$
training points from the innocuous corpus, calculate the center of mass
and fix the radius such that the false positive rate on the training data is
$\alpha=0.001$. Then we draw a random instance from each of the 20
attack classes, and for each of these 20 attack instances generate a
poisoning attack as described in Section~\ref{subsec:ids-greedy}. An
attack succeeds when the attack point is accepted as innocuous by a
learning algorithm.

For each attack instance, the number of iterations needed for an
attack to succeed and the respective displacement of the center of
mass is recorded. Figure~\ref{fig:ids-evasion} shows, for each attack
instance, the behavior of the relative displacement at the point of
success as a function of a number of iterations. We interpolate a
``displacement curve'' from these pointwise values by a linear
least-squares regression. For comparison, the theoretical upper bounds
for the average-out and all-in cases are shown. Notice that the bound
for the all-in strategy is also almost linear for the small $i/n$
ratios observed in this experiment.

\begin{figure}
\begin{center}
\includegraphics[width=0.6\textwidth]{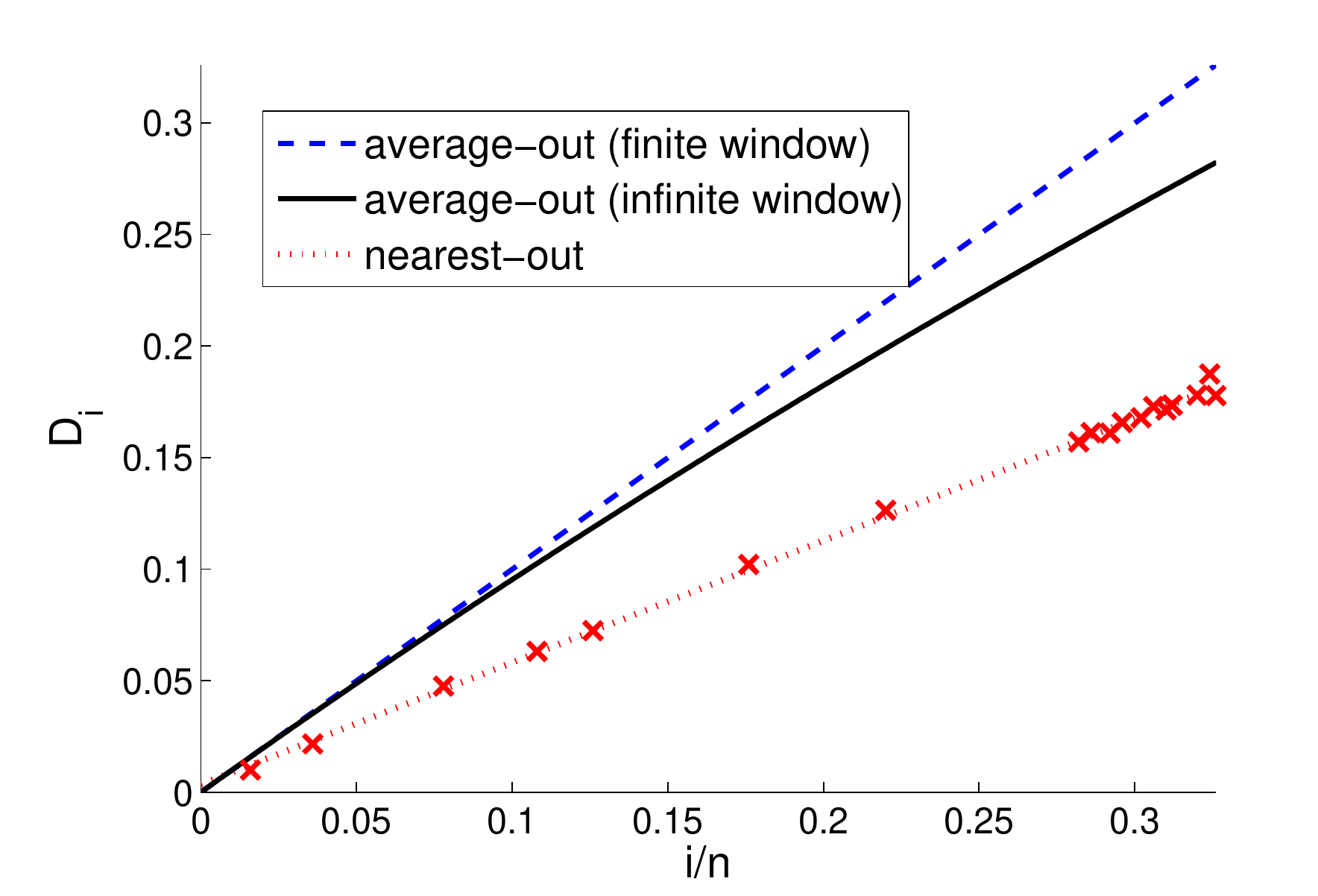}
\end{center}
\caption{Empirical displacement of the nearest-out centroid for $20$
  different exploits (crosses, linear fit shown by a red dotted line).
  Displacement values are shown at the point of success for each
  attack.  Theoretical bounds are shown for comparison (blue and black
  lines).}
\label{fig:ids-evasion}
\end{figure}

The observed results confirm that the linear progress rate in the full
control scenario can be attained in practice for real data. Compared to the
simulations of Section~\ref{subsec:ids-greedy}, the progress rate of
an attack is approximately half the one for the average-out
case. Although this somewhat contradicts our expectation that for a
high-dimensional space (of the effective dimensionality $d\sim 256$ as it was found in
Section~\ref{subsec:dim}) the progress rate to the average-out case
should be observed, this can be attributed to multiple approximations
performed in the generation of an attack for real byte sequences. The
practicality of a poisoning attack is further emphasized by a small
number of iterations needed for an attack to succeed: from 0 to only 35 percent
of the initial number of points in the training data have to be overwritten by an attacker.

%In Fig.~\ref{fig:ids-greedy2} we compared the regression curve to the average-out bound and the Nelson bound of Th.~\ref{th:infinite}.
%
%\begin{figure}
%\includegraphics[width=0.75\textwidth]{real2}
%\caption{Evasion of online centroid learner trained on real IDS data. Nelson and average-out bounds are compared with the results of experiments 1 and 2.}
%\label{fig:ids-greedy2}
%\end{figure}

\subsection{Critical Traffic Ratios of HTTP Attacks}\label{sec:nu-crit}

For the case of attacker's limited control, the success of the
poisoning attack largely depends on attacker's constraints, as shown
in the analysis in Sections~\ref{sec:limited} and
\ref{sec:defense}. The main goal of the experiments in this section is
therefore to investigate the impact of potential constraints in
practice. In particular, we are interested in the impact of the traffic
ratio $\nu$ and the false positive rate $\alpha$.

The analysis in Section~\ref{sec:limited}
(cf. Theorem~\ref{th:limited} and Figure~\ref{fig:bound_limited})
shows that the displacement of a poisoning attack is bounded from above by
a constant, depending on the traffic ratio $\nu$ controlled by an
attacker. Hence the susceptibility of a learner to a particular attack
depends on the value of this constant. If an attacker does not control
a sufficiently large traffic portion and the potential displacement is
bounded by a constant smaller than the distance from the initial
center of mass to the attack point, then an attack is bound to
fail. To illustrate this observation, we compute critical traffic
rates needed for the success of each of the 20 attack classes in our
malicious pool.

We randomly draw a $1000$-elemental training set from the innocuous pool
and calculate its center of mass (in the space of 3-grams). The radius
is fixed such the false positive rate $\alpha=0.001$ on innocuous data is
attained. For each of the 20 attack classes we compute the class-wise
median distance to the centroid's boundary. Using these distance
values we calculate the ``critical value'' $\nu_{\text{crit}}$ by
solving Th.~\ref{th:limited}(c) for $\nu$ (cf. Eq.~\eqref{eq:nu-crit}). 
The experiments have been repeated $10$ times
results are shown in Table~\ref{tab:crit-nu}. 

\begin{table}[h]
\small
\centering
\begin{tabular}{lcc}
  \textbf{Attacks} & \textbf{Rel. dist.} & $\boldsymbol\nu_{\textbf{crit}}$ \\
  
  ALT-N WebAdmin Overflow & $0.058\pm 0.002$ & $0.055\pm 0.002$ \\ ApacheChunkedEncoding & $0.176\pm 0.002$ & $0.150\pm 0.001$ \\
  AWStats ConfigDir Execution & $0.067\pm 0.002$ & $0.063\pm 0.002$ \\ Badblue Ext Overflow & $0.168\pm 0.002$ & $0.144\pm 0.001$ \\
  
  Barracuda Image Execution & $0.073\pm 0.002$ & $0.068\pm 0.002$ \\ Edirectory Host & $0.153\pm 0.002$ & $0.132\pm 0.001$ \\
  IAWebmail & $0.178\pm 0.002$ & $0.151\pm 0.001$ \\ IIS 5.0 IDQ exploit & $0.162\pm 0.002$ & $0.140\pm 0.001$ \\
  
  Pajax Execute & $0.107\pm 0.002$ & $0.097\pm 0.002$ \\ PEERCAST URL & $0.163\pm 0.002$ & $0.140\pm 0.001$ \\
  PHP Include & $0.097\pm 0.002$ & $0.088\pm 0.002$ \\ PHP vBulletin & $0.176\pm 0.002$ & $0.150\pm 0.001$ \\
  
  PHP XML RPC  & $0.172\pm 0.002$ & $0.147\pm 0.001$ \\ HTTP tunnel & $0.160\pm 0.002$ & $0.138\pm 0.001$ \\
  IIS 4.0 HTR exploit & $0.176\pm 0.002$ & $0.149\pm 0.002$ \\ IIS 5.0 printer exploit & $0.161\pm 0.002$ & $0.138\pm 0.001$ \\
  
  IIS unicode attack & $0.153\pm 0.002$ & $0.133\pm 0.001$ \\ IIS w3who exploit & $0.168\pm 0.002$ & $0.144\pm 0.001$ \\
  IIS 5.0 WebDAV exploit & $0.179\pm 0.002$ & $0.152\pm 0.001$ \\ rproxy exploit & $0.155\pm 0.002$ & $0.134\pm 0.001$ \\
  
\end{tabular} 
\caption{Relative distances (in radii) of exploits to the boundary of a
  centroid enclosing all training points and critical values of
  parameter $\nu$.}
\label{tab:crit-nu}
\end{table}

The results indicate that in order to subvert a online centroid learner an
attacker needs to control from 5 to 20 percent of traffic. This could
be a significant limitation on highly visible sites. 
Note that an attacker usually aims at earning money by hacking computer systems.
However generating competitive bandwidths at highly visible site 
is likely to drive the attacker's cost to exorbitant numbers.

On the other
hand, one can see that the traffic rate limiting alone cannot be seen
as sufficient protection instrument due to its passive nature. In the
following section we investigate a different protection scheme using
both traffic ratio and the false positive rate control.

\subsection{Poisoning Attack against Learner with False Positive Protection}

The analysis in Section~\ref{sec:limited}
(cf. Theorem~\ref{th:limited} and Figure~\ref{fig:bound_limited})
shows that the displacement of a poisoning attack is bounded from above by
a constant, depending on a traffic ratio $\nu$ and a maximal false positive rate
$\alpha$. Hence a detection system can be protected by 
observing the system's false positive rate and 
switching off the online updates if a defined threshold is
exceeded.

\subsubsection{Experiment 1: Practicability of False Positive Protection}

\begin{figure}
\begin{center}
\includegraphics[width=0.65\textwidth]{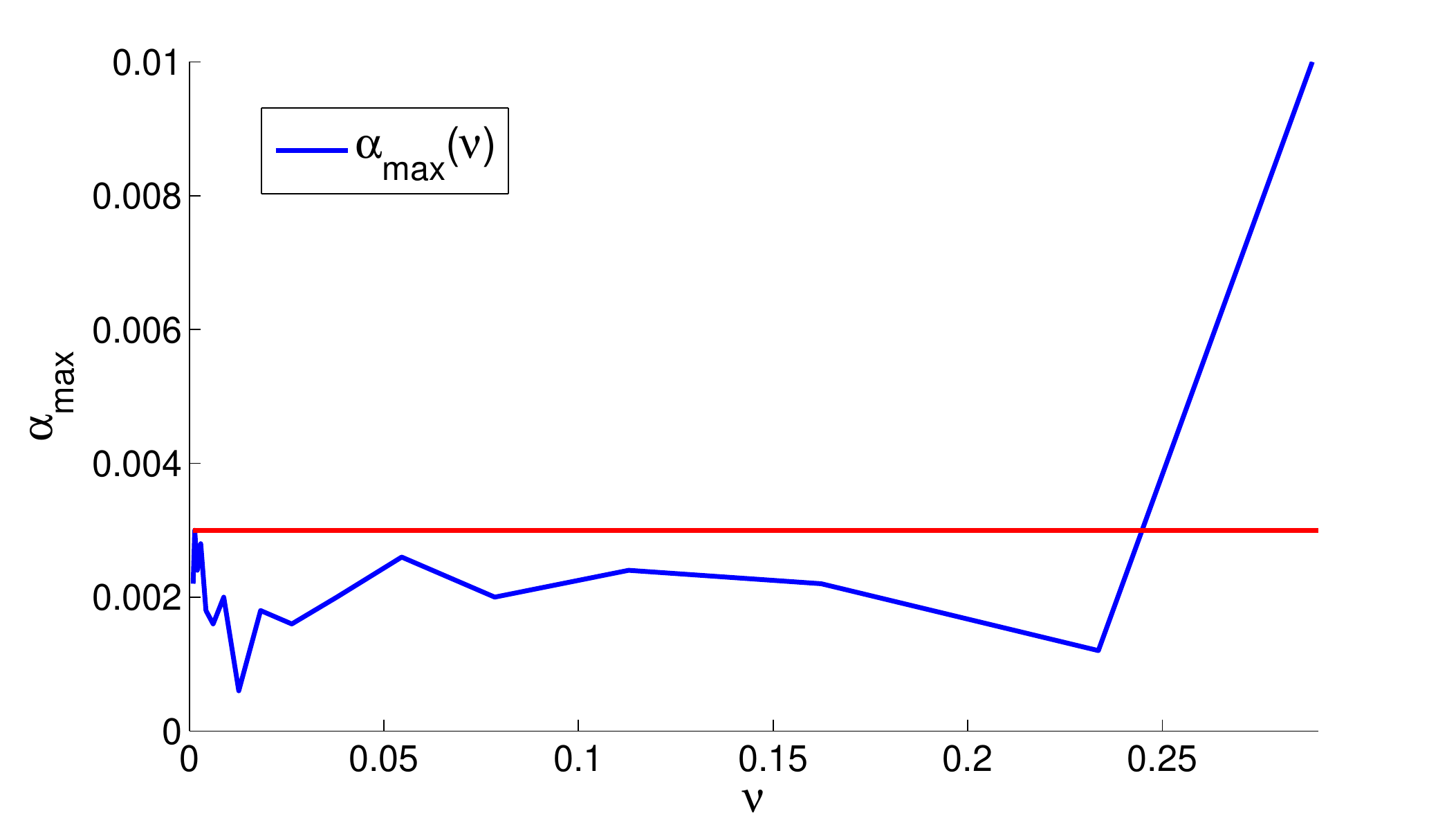}
\end{center}
\caption{Maximal false positive rate within 10000 attack iterations as
  a function of $\nu$ (maximum taken over 10 runs). }
\label{fig:alpha_max}
\end{figure}

However in practice the system should be as silent as possible, i.e., 
an administrator should be only alarmed if a fatal danger to 
the system is given. We hence in this
section investigate how sensible the false positive rate
is to small adversarial perturbations of the learner, caused by 
poisoning attack with small $\nu$.

Therefore the following experiment investigates 
the rise in the false positive rate $\alpha$ as a function of $\nu$.
From the innocuous pool we randomly drew a $1000$-elemental training set
on base of which a centroid is calculated. Thereby the radius is fixed
to the empirical estimate of the $0.001$-quantile of the innocuous pool
based on $100$ randomly drawn subsamples, i.e., we expect the
centroid having a false positive rate of $\alpha=0.001$ on the innocuous
pool. Moreover we randomly drew a second $500$-elemental training set
from the innocuous pool which is reserved for online training and
and a $500$-elemental hold out set on base of which a false positive rate can be
estimated for a given centroid.
Then we iteratively calculated poisoning attacks
with fixed IIS 5.0 WebDAV exploit as attack point by subsequently presenting 
online training points to the centroid learner which are rejected or
accepted based on whether they fall within the learner's radius.
For each run of a poisoning attack the false positiv rate is observed on base of the
hold out set.

In Fig.~\ref{fig:alpha_max} we plot for various values of $\nu$
the maximal observed false positive rate as a function of $\nu$, 
where the maximum is taken over all
attack iterations and $10$ runs.
One can see from the plot that $\alpha=0.005$ is a reasonable
threshold in our setting to ensure the systems's silentness. 

\subsubsection{Experiment 2: Attack Simulation for False Positive Protection}

\begin{figure}[bp]
  \centering
  \vspace*{-0.4cm}
  \includegraphics[width=0.6\linewidth]{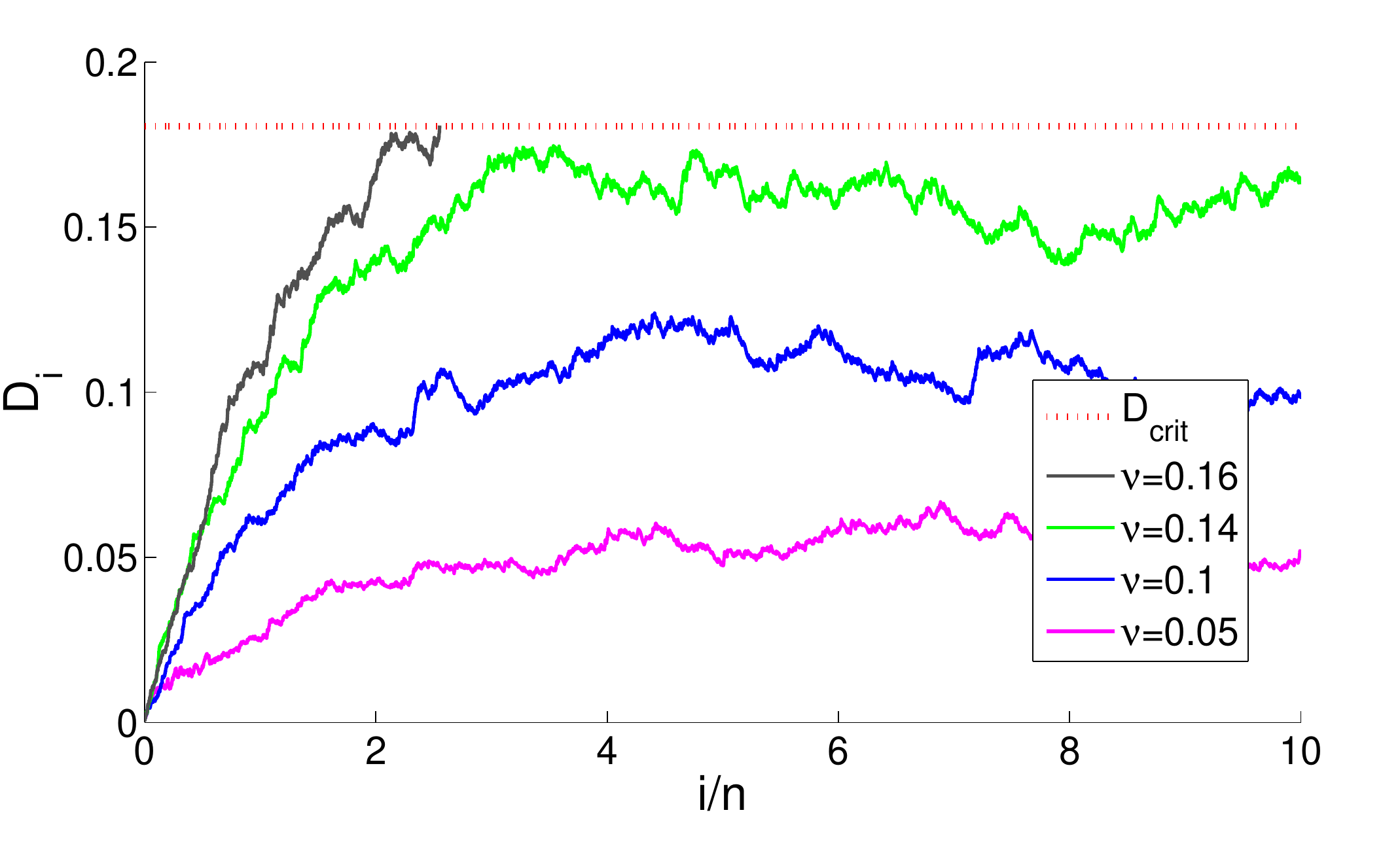}
  \vspace*{-0.4cm}
  \caption{A simulation of a poisoning attack under limited control.}
  \label{fig:attack_sim}
  %\vspace*{-0.4cm}
\end{figure}

In the previous experiment we have seen that $\alpha=0.005$ is a reasonable threshold
for a false positive protection to ensure a systems silentness. 
We in this section illustrate that the critical values from Section~\ref{sec:nu-crit} 
computed on base of Th.~\ref{th:limited} 
for maximal false positive rate of $\alpha=0.005$ still give a good approximation of
the true impact of a poisoning attack.

We fix a particular exploit in our malicious
corpus (IIS WebDAV 5.0 exploit) and run a poisoning attack against
the average-out centroid for various values of $\nu\in[0.05, 0.10, 0.14, 0.16]$,
recording the actual displacement curves. One can see from Fig.~\ref{fig:attack_sim}
that the attack succeeds for $\nu = 0.16$ but fails to reach the
required relative displacement of $D_{crit}=0.18$ for $\nu =
0.14$. The theoretically computed critical traffic ratio for
this attack according to Table~\ref{tab:crit-nu} 
is $\nu_{\rm crit} =0.152$.  The experiment shows that the derived bounds
are surprisingly tight in practice.

\subsubsection{Implementation of Poisoning Protection}

In Section~\ref{sec:limited} we have seen, that an attacker's
impact on corrupting the training data highly depends on the fraction
of adversarial points in the training data stream. This implies that a
high amount of innocuous training points constantly has to come in. In
Section~\ref{sec:defense} we have seen, that we can secure the learner
by setting a threshold on the false positive rate $\alpha$. Exceeding
the latter enforces further defense processes such as switching off
the online training process. Hence an confident estimation of $\alpha$
has to be at hand. How can we achieve the latter?

In practice, this can e.g. be done by caching the training data. When
the cache exceeds a certain value at which we have a confident
estimation of $\alpha$ (e.g., after 24 hours), the cached training data
can be applied to the learner. Since in applications including
intrusion detection, we usually deal with a very high amount of
training data, a confident estimation is already possible after short
time period.

\section{Discussion and Conclusions}

Understanding of security properties of learning algorithms is
essential for their protection against abuse. The latter can take
place when learning is used in applications with competitive interests
at stake, e.g., security monitoring, games, spam protection, reputation
systems, etc. Certain security properties of a learning algorithm must
be \emph{proved} in order to claim its immunity to abuse. To this
end, we have developed a methodology for security analysis and applied
it for a specific scenario of online centroid anomaly detection. The
results of our analysis highlight conditions under which an attacker's
effort to subvert this algorithm is prohibitively high.

Several issues discussed in this contribution have appeared in related work
albeit not in the area of anomaly detection.  Perhaps the most
consummate treatment of learning under an adversarial impact has been
carried out by \citet{DalDomMauSanVer04}. In this work, Bayesian
classification is analyzed for robustness against adversarial
impact. The choice of their classifier is motivated by widespread
application of the naive Bayes classification in the domain of spam
detection where real examples of adversarial impact have been observed
for a long time. The adversarial classification is considered as a
game between an attacker and a learner. Due to the complexity of
analysis, only one move by each party can be analyzed. Similar to our
approach, \cite{DalDomMauSanVer04} formalize the problem by defining
cost functions of an attacker and a learner (Step 1) and determine an
optimal adversarial strategy (Step 3).  Although the attacker's
constraints are not explicitly treated theoretically, several
scenarios using specific constraints have been tested
experimentally. No analysis of the attacker's gain is carried out;
instead, the learner's direct response to adversarial impact is
considered.

A somewhat related approach has been developed for handling worst-case
random noise, e.g., random feature deletion
\citep{GloRow06,DekSha08}. Similar to \cite{DalDomMauSanVer04}, both
of these methods construct a classifier that automatically reacts to
the worst-case noise or, equivalently, the optimal adversarial
strategy. In both methods, the learning problem is formulated as a
large-margin classification using a specially constructed risk
function. An important role in this approach is played by the
consideration of constraints (Step 2), e.g., in the form of the maximal
number of corruptible features. Although these approaches do not
quantitatively analyze attacker's gain, \citep{DekSha08} contains an
interesting learning-theoretic argument that relates classification
accuracy, sparseness, and robustness against adversarial noise.

To summarize, we believe that despite recent evidence of possible
attacks against machine learning and the currently lacking theoretical
foundations for learning under adversarial impact, machine learning
algorithms \emph{can} be protected against such impact. The key to
such protection lies in quantitative analysis of security of machine
learning. We have shown that such analysis can be rigorously carried
out for specific algorithms and attacks. Further work should extend
such analysis to more complex learning algorithms and a wider attack
spectrum.

\begin{acks}
The authors wish to thank Ulf Brefeld, Konrad Rieck, Vojtech Franc, Peter Bartlett and Klaus-Robert M\"uller for fruitful
discussions and helpful comments. Furthermore we thank Konrad Rieck for providing the network traffic.
This work was supported in part by the 
German Bundesministerium f\"ur Bildung und Forschung (BMBF) under the project REMIND (FKZ 01-IS07007A), by the German Academic Exchange Service, 
and by the FP7-ICT Programme of the European Community, under the PASCAL2  Network of Excellence, ICT-216886. 
\end{acks}

\appendix
\normalsize

\section{Notation Summary}
In this paper we use the following notational conventions.
 \begin{table*}[h]
 \begin{center}
 \begin{tabular}{lp{7cm}}
 $\mathcal C, r, \v c$ & centroid $\mathcal C$ with radius $r$ and center $\v c$\\
 $i$ & $i$-th attack iteration, $i\in\mathbb N_0$ \\
 $\v x_i, \v X_i$ & center of centroid in $i$-th attack iteration\\
 $\v A$ & attack point \\
 $\v a$ & attack direction vector \\
 \end{tabular} 
 \end{center}
\end{table*}

\begin{table*}[h]
\begin{center}
\begin{tabular}{lp{7cm}}
  $ D_i $ &  $i$-th relative displacement of a centroid in radii into direction of $\v a$ \\
  $n$ & number of training patterns of centroid\\
 $f$ & function of $\mathcal H\rightarrow\mathcal H$ giving an attack strategy \\
 $\nu$ & fraction of adversarial training points \\
 $B_i$ & Bernoulli variable \\
 $\epsilon_i, \vepsilon_i$ & i.i.d. noise \\
 $\alpha$ & false alarm rate \\
 $I_S$ & indicator function of a set $S$ \\
 \end{tabular} 
 \end{center}
\end{table*}

\section{Auxiliary Material and Proofs}\label{app:proofs}

\subsection{Auxiliary Material for Section \ref{sec:full}}
\label{subsec:greedy-theory}

\subsubsection{Representer Theorem for Optimal Greedy Attack}

First, we show why the attack efficiency cannot be increased beyond 
dimensions with $d \geq n+1$. This follows from the fact that the optimal
attack lies in the span of the working set points and the attack
vector. The following representer theorem allows for ``kernelization'' of the optimal 
greedy attack.

\begin{theorem}\label{th:span}
  There exists an optimal solution of problem~(\ref{eq:greedy-final}) satisfying
  \begin{equation}\label{eq:span}
    \v x^{*}_i\in~{\rm span}(\v a,\v x_1,...,\v x_n).
  \end{equation}
\end{theorem}

\begin{proof}
  The Lagrangian of optimization problem \eqref{eq:greedy-final}
  is given by:
  \begin{eqnarray*}
    L(\v x,\valpha,\beta) & = & -(\v x-\v x_i)\cdot\v a 
    + \sum_{j=1}^n \alpha_j\left(2(\v x_j-\v x_i)\cdot \v x-\v x_j\cdot\v x_j 
      + \v x_i\cdot\v  x_i\right) \\
    & & + \beta\left(\v x\cdot\v x -\frac{2}{n}\sum_{j=1}^n\v x\cdot\v x_j 
      +\frac{1}{n^2}\sum_{j,k=1}^n\v x_j\cdot\v x_k -r^2 \right)
  \end{eqnarray*}
  Since the feasible set of problem~(\ref{eq:greedy-final}) is bounded by the spherical constraint and is not empty ($\v x_i$ trivially is contained in the feasible set),  there exists at least one optimal solution $\v x_i^{*}$ to the primal. For optimal $\v x_i^{*}$, $\alpha^{*}$ and $\beta^{*}$, we have the following first order optimality conditions
  \begin{equation}\label{eq:opt_cond}
    \frac{\delta L}{\delta\v x} = 0: ~ ~ ~ ~ -\v a 
    -\frac{1}{n}\sum_{j=1}^n\v x_j +2\sum_{j=1}^n\alpha_j^{*}(\v x_j-\v x_i) 
    +\beta^{*}\left(2\v x_i^{*} -\frac{2}{n}\sum_{j=1}^n\v x_j\right) = 0 ~ . 
  \end{equation}
  If $\beta^{*}\neq 0$ the latter equation can be resolved for $\v
  x_i^{*}$ leading to:
  $$\v x_i^{*} =  \frac{1}{2\beta^{*}}\v a + \sum_{j=1}^n\left(\frac{1}{2\beta^{*} n} 
    - \frac{\alpha_j^{*}}{\beta^{*}} + \frac{1}{n}\right)\v x_j +
  \frac{1}{\beta^{*}}\sum_{j=1}^n\alpha_j^{*}\v x_i ~ .$$ From the
  latter equation we see that $\v x$ is contained in $S:=\text{span}(\v
  x_1,...,\v x_n$ and $\v a)$.

  Now assume $\beta^*=0$ and $x_i^{*}\notin S$. Basically the idea of the following reasoning is to use $x_i^{*}$ to construct an optimal point which is contained in $S$. At first, since $\beta^*=0$, we see from Eq.~\eqref{eq:opt_cond} that $\v a$ is contained in the subspace $S:=\text{span}(\v x_1,...,\v x_n)$. Hence the objective, $(\v x-\v x_i)\cdot\v a$, only depends on the optimal $\v x$ via inner products with the data $\v x_i$. The same naturally holds for the constraints. Hence  both, the objective value and the constraints, are invariant  under the projection of $x_i^{*}$ onto $S$, denoted by  $P$. Hence $P(x_i^{*})$ also is an optimal point. Moreover by construction $P(x_i^{*})\in S=\text{span}(x_1^*,...,\v x_n^*)$.
  
  %If the optimal point is unique, this only
%  can happen when the $\v x_i$ span the whole input space. If the
%  optimal solution is not unique and we may without loss of generality
%  take a solution which is spanned by $\v x_1,...,\v x_n$ and $\v a$.
%  \comment{I do not get this argument. In my understanding, if $\beta
%    = 0$, we have a linear program which always has a solution lying
%    in the subspace spanned by constraints. Since constraint are
%    imposed by individual points, we get the right span. If the set of
%    constraints is linearly independent, the full space is spanned,
%    i.e. the dimension is $n$. Otherwise mutliple the dimension is
%    less than $n$. I am also not sure if multiple solutions are
%    possible for linear problems...}
\end{proof}

\subsubsection{Theoretical Analysis for the Optimal Greedy Attack}

The dependence of an attack's effectiveness on the data dimensionality
results from the geometry of Voronoi cells. Intuitively, the
displacement at a single iteration depends on the size of the largest
Voronoi cell in a current working set. Although it is hard to derive a
precise estimate on the latter, the following ``average-case''
argument sheds some light on the attack's behavior, especially since
it is the average-case geometry of the working set that determines the
overall -- as opposed to a single iteration -- attack progress.

Consider a simplified case where each of the Voronoi cells $C_j$
constitutes a ball of radius $r$ centered at a data point $\v x_j$, $j
= 1, \ldots, n$. Clearly, the greedy attack will results in a progress
of $r/n$ (we will move one of the points by $r$ but the center's
displacement will be discounted by $1/n$). We will now use the
relationships between the volumes of balls in $\mathbb R^d$ to relate $r$,
$R$ and $d$.

The volume of each Voronoi cell $C_j$ is given by 
$$
{\rm Vol}(C_j) = \frac{ \pi^{\frac{d}{2}}r^d }
{\Gamma\left(\frac{d}{2}+1\right)} ~ .
$$
Likewise, the volume of the hypersphere $S$ of radius $R$ is
$$
{\rm Vol}(S) = \frac{ \pi^{\frac{d}{2}}R^d }
{\Gamma\left(\frac{d}{2}+1\right)} ~ .
$$
Assuming that the Voronoi cells are ``tightly packed'' in $S$, we
obtain 
$$
{\rm Vol}(S) \approx n \,{\rm Vol}(C_j).
$$
Hence we conclude that
$$
r \approx \sqrt[d]{\frac{1}{n}}~R.
$$
One can see that the attacker's gain, approximately represented by the
cell radius $r$, is a constant fraction of the threshold $R$, which
explains the linear progress of the poisoning attack. The slope of
this linear dependence is controlled by two opposing factors: the size
of the training data decreases the attack speed whereas the intrinsic
dimensionality of the feature space increases it. Both factors depend
on fixed parameters of the learning problem and cannot be controlled by
an algorithm. In the limit, when $d$ approaches $n$ (the effective
dimension is limited by the training data set according to
Th.~\ref{th:span}) the attack progress rate is approximately
described by the function $\sqrt[n]{\textstyle\frac{1}{n}}$ which
approaches 1 with increasing $n$.

%\begin{figure}
%\includegraphics[width=0.75\textwidth]{squarerootterm}
%\caption{Asymptotic behavior of radii of Voronoi cells.}
%\label{fig:greedy-theory}
%\end{figure}

\subsection{Proofs of
Section~\ref{sec:limited}}\label{app:proofs-limited}

\begin{proposition}{\rm \textbf{(Geometric series)}}\label{prop:geom-series} ~
Let $(s)_{i\in\mathbb N_0}$ be a sequence of real numbers satisfying 
$s_0=0$ and $~ s_{i+1} = qs_i + p ~$ (or $~ s_{i+1} \leq qs_i + p ~$ or
$~ s_{i+1} \geq qs_i + p ~$) for some $p,q>0 ~ .$ Then it holds:
\begin{equation}\label{eq:geom-series}
  s_i = p\frac{1-q^{i}}{1-q} ~, ~ ~ ({\rm and} ~ ~ s_i \leq p\frac{1-q^{i}}{1-q} ~ ~ {\rm or}  s_i \geq p\frac{1-q^{i}}{1-q} )~,
\end{equation}
respectively.
\end{proposition}

\begin{proof}\hspace{0pt}

(a) ~ We prove part (a) of the theorem by induction over $i\in\mathbb N_0$, 
the case of $i=0$ being obvious. 
  
In the inductive step we show that if Eq.~\eqref{eq:geom-series} holds for an 
arbitrary fixed $i$ it also holds for $i+1$:
\begin{eqnarray*}
  s_{i+1} & = & qs_i + p = q\left(p\frac{1-q^i}{1-q}\right)+p 
  =  p\left(q\frac{1-q^i}{1-q}+1\right) \\
  & = &  p\left(\frac{q-q^{i+1} + 1-q}{1-q}\right) 
  = p\left(\frac{1-q^{i+1}}{1-q}\right) ~ .
\end{eqnarray*}

(b) ~ The proof of part (b) is analogous.
\end{proof}

%\textbf{Proof of Prop.~\ref{prop:limited-optattack}.}
%\begin{proof}
%  At first, we observe that by Axiom~\ref{axiom:limited}, we have the
%  constraint
%$$ \Vert f(x)-x\Vert \leq r ~ .$$
%Hence any valid attack strategy can be written as 
%\begin{equation}\label{eq:attack-decomp}
%  f(x) = x + g(x) ~ ~ \text{with} ~ ~ \Vert g\Vert \leq r ~ , ~ R=1 .
%\end{equation}
%It holds:
%\begin{eqnarray*}
%  D_{i+1} &\stackrel{\text{Def.}~\ref{def:displ}}{\leq} & 
%  \v X_{i+1}\cdot\v a \\
%  & \stackrel{\text{Ax.}~\ref{axiom:limited}}{=} &  
%  \left(\v X_i +\frac{1}{n}\left(B_if(\v X_i) + (1-B_i)\vepsilon_i 
%  - \v X_i\right)
%  \right) \cdot \v a\\
%  & \stackrel{\eqref{eq:attack-decomp}}{=} &  \left(\v X_i 
%  +\frac{1}{n}\left(B_i\left(\v X_i+g(\v X_i)\right) 
%  + (1-B_i)\vepsilon_i - \v X_i\right)\right)\cdot \v a \\
%  &=& \frac{1}{n}B_ig(\v X_i)\cdot \v a +  D_i +\frac{1}{n}\left(B_i D_i  + (1-B_i)\epsilon_i - D_i\right)
%\end{eqnarray*}
%
%
%The optimal attack strategy $f(x)=x+g(x)$ is by definition a maximizer of above term. 
%Since the right summand does not depend on $g$, this is equivalent to maximizing the left term, i.e. 
%$\frac{1}{n}B_ig(\v X_i)$.
%Maximizing the
%latter with respect to $g$ corresponds to maximizing $g(X_i)\cdot\v
%a$ (note $B_i\geq 0$), hence by the  Cauchy-Schwarz inequality we have the maximizer 
%$g(X_i) = \lambda\v a$, furthermore by \eqref{eq:attack-decomp} $\lambda=1$, thus $f(x)=x+\v a$.
%\end{proof}

\begin{proof}\textbf{of Th.~\ref{th:limited}(b)} ~Multiplying both sides of Eq.~\eqref{eq:lim-main-th} with $\v a$ and 
substituting $D_i = \v X_i\cdot\v a$ results in
$$D_{i+1} = \left(1-\frac{1-B_i}{n}\right)D_i + \frac{B_i}{n} 
  +\frac{(1-B_i)}{n}\vepsilon_i\cdot\v a ~ .$$
Inserting $B_i^2=B_i$ and $B_i(1-B_i)=0$, which holds because $B_i$ is
Bernoulli, into the latter equation, we have:
$$D_{i+1}^2 = \left(1-2\frac{1-B_i}{n} + \frac{1-B_i}{n^2}\right)D_i^2 
  +\frac{B_i}{n^2} + \frac{(1-B_i)}{n^2}\Vert\vepsilon_i\cdot 
  \v a\Vert^2 +2\frac{B_i}{n}D_i   
  +2(1-B_i)(1-\frac{1}{n})D_i\vepsilon_i\cdot\v a ~ .$$
Taking the expectation on the latter equation, and noting that by
Axiom~\ref{axiom:limited} $\vepsilon_i$ and $\v D_i$ are independent, we have:
\begin{eqnarray}\label{eq:2nd-momentaux}
  E\left(D_{i+1}^2\right) &=& 
  \left(1-\frac{1-\nu}{n}\left(2-\frac{1}{n}\right)\right)E\left(D_i^2\right) 
  + 2\frac{\nu}{n}E(D_i) + \frac{\nu}{n^2} 
  + \frac{1-\nu}{n^2}E(\Vert\vepsilon_i\cdot\v a\Vert^2) \nonumber\\
  & \stackrel{(1)}{\leq} &  
  \left(1-\frac{1-\nu}{n}\left(2-\frac{1}{n}\right)\right)E\left(D_i^2\right) 
  + 2\frac{\nu}{n}E(D_i) + \frac{1}{n^2} 
\end{eqnarray}
where (1) holds because by Axiom~\ref{axiom:limited} we have 
$\Vert\vepsilon_i\Vert^2\leq r$ and by Def.~\ref{def:displ} 
$\Vert\v a\Vert=R$, $R=1$.
Inserting the result of (a) in the latter equation  results in the
following recursive formula:
$$ E\left(D_{i+1}^2\right) \leq  
  \left(1-\frac{1-\nu}{n}\left(2-\frac{1}{n}\right)\right)E\left(D_i^2\right) 
  + 2(1-c_i)\frac{\nu}{n}\frac{\nu}{1-\nu} + \frac{1}{n^2} ~ .$$
By the formula of the geometric series, i.e., by Prop.\ref{prop:geom-series}, we have: 
$$
  E\left(D_{i}^2\right) \leq \left(2(1-c_i)\frac{\nu}{n}\frac{\nu}{1-\nu}
  +\frac{1}{n^2}\right)\frac{1-d_i}{\frac{1-\nu}{n}\left(2-\frac{1}{n}\right)} ~ ,
$$
denoting $d_i:=\left(1-\frac{1-\nu}{n}\left(2-\frac{1}{n}\right)\right)^i$. Furthermore by some algebra
\begin{eqnarray}\label{eq:2nd-moment}
  E\left(D_{i}^2\right)
  \leq \frac{(1-c_i)(1-d_i)}{1-\frac{1}{2n}}\frac{\nu^2}{\left(1-\nu\right)^2} 
  + \frac{1-d_i}{(2n-1)(1-\nu)}.
\end{eqnarray}
We will need the auxiliary formula
\begin{equation}\label{eq:aux-form}
  \frac{(1-c_i)(1-d_i)}{1-\frac{1}{2n}}-(1-c_i)^2 \leq \frac{1}{2n-1}+c_i-d_i ~ ,
\end{equation}
which can be verified by some more algebra and employing $d_i<c_i$.
We finally conclude
\begin{eqnarray*}
  {\rm Var}(D_i) & =& E(D_i^2) - (E(D_i))^2 \\
  &\stackrel{\text{Th.}\ref{th:limited-main}(a);~\text{Eq.}\eqref{eq:2nd-moment}}{\leq}&  \left(\frac{(1-c_i)(1-d_i)}{1-\frac{1}{2n}} -(1-c_i)^2\right)\left(\frac{\nu}{1-\nu}\right)^2 
  + \frac{1-d_i}{(2n-1)(1-\nu)^2} \\
  &\stackrel{\text{Eq.}\eqref{eq:aux-form}}{\leq}& \gamma_i \left(\frac{\nu}{1-\nu}\right)^2 +\delta_n
\end{eqnarray*}
%((1-c_i)-(1-c_i)^2)\left(\frac{\nu}{1-\nu}\right)^2 
%  + \delta \stackrel{(3)}{\leq} c_i\left(\frac{\nu}{1-\nu}\right)^2 + \delta ~ ,$$
where $\gamma_i:=c_i-d_i$ and $\delta_n:=\frac{\nu^2 +(1-d_i)}{(2n-1)(1-\nu)^2}$.
This completes the proof. 

%(c) and (d) are easily derived from (a) and (b) by noting that 
%$0\leq c_i<1$, $c_i\rightarrow 1$ for $i\rightarrow\infty$, hence $\gamma_i\rightarrow 0$, moreover
%$\delta(n)\rightarrow 0$ for $n\rightarrow\infty$.
\end{proof}

\subsection{Proofs of Section~\ref{sec:defense}}\label{app:defense}

%Proof \textbf{of Prop.~\ref{prop:defense-optattack}} ~
%\begin{qedproof}
%Let $\v a $ be an attack vector. By Ax.~\ref{axiom:defense} we have 
%\begin{equation}\label{eq:prop-optattack}
%  \v X_{i+1} \leq \v X_i + \frac{1}{n}\left(B_i\left(f(\v X_i)-\v X_i\right) 
%  + (1-B_i)I_{\{\Vert \vepsilon_i-\v X_i\Vert\leq r\}}\left(\vepsilon_i 
%  - \v X_i\right)\right) 
%\end{equation}
%Multiplying $\v a$ on both sides of above equation and taking the expected
% values, since by definition $D_i=\v X_i\cdot \v a$, results in:
%$$  E(D_{i+1}) \leq \left(1-\frac{\nu}{n}\right)E(D_i) 
%  + \frac{\nu}{n}E(f(\v X_i)\cdot\v a) 
%  + \frac{1-\nu}{n}E\left(I_{\{\Vert \vepsilon_i-\v X_i\Vert\leq r\}}
%  \left(\vepsilon_i-D_i\right)\right) ~ .$$
%Since by Ax.~\ref{axiom:defense} we have the constraint 
%$\Vert f(x)-x\Vert\leq r=1$, we by the 
%Cauchy-Schwartz inequality conclude that the right hand side of the latter equation
%is maximized with respect to $f$ by $f(\v X_i)= \v X_i+\v a$. Now the
%statement follows by plugging the latter optimal attack strategy into
%Eq.~\eqref{eq:prop-optattack}, multiplying with $\v a$ on both sides and 
%substituting $D_i=\v X_i\cdot \v a$, ~ $\v a\cdot\v a =R$, ~ $R=1$ and 
%$I_{\{\Vert \vepsilon_i-\v X_i\Vert\leq r\}} 
%  = 1-I_{\{\Vert \vepsilon_i-\v X_i\Vert>r\}}$.
%\end{qedproof}
 
\begin{lemma}\label{lemma:aux-bound}
Let $\mathcal C$ be a protected online centroid learner satisfying the optimal 
attack strategy. Then we have:
\begin{eqnarray*}  
  &{\rm (a)} ~ ~ ~ & 0\leq  
  E\left(I_{\{\Vert \vepsilon_i- X_i\Vert>r\}}D_i^q\right)  \leq \alpha E(D_i^q) ~, \quad q=1,2   \\
  &{\rm (b)} ~ ~ ~ & 0\leq  E\left(I_{\{\Vert \vepsilon_i- X_i\Vert\leq r\}}
  \epsilon_i\right)  \leq \alpha   \\
  &{\rm (c)} ~ ~ ~ & E\left(I_{\{\Vert \vepsilon_i- X_i\Vert\leq r\}}\epsilon_iD_i\right) \leq \alpha E(D_i) ~ .
\end{eqnarray*}
\end{lemma}

\begin{proof}

(a) ~ Let $q=1$ or $q=2$. Since $\epsilon_i$ is independent of $\v X_i$ (and hence of $D_i$), we have
$$   E_{\vepsilon_i}\left(I_{\{\Vert \vepsilon_i- X_i\Vert>r\}}D_i^q\right) 
   = (D_i)^q  E_{\vepsilon_i}\left(I_{\{\Vert \vepsilon_i- X_i\Vert>r\}}\right) ~ .$$
 Hence by Ax.~\ref{axiom:defense} 
$$\text{  $E_{\vepsilon_i}\left(I_{\{\Vert \vepsilon_i- X_i\Vert>r\}}D_i^q\right) = 0$
 ~ if ~ $e(\v X_i):=E_{\vepsilon_i}\left(I_{\{\Vert \vepsilon_i- \v X_i\Vert>r\}}\right)>\alpha$,}$$
 and 
$$\text{$0\leq E_{\vepsilon_i}\left(I_{\{\Vert \vepsilon_i- X_i\Vert>r\}}D_i^q\right) \leq\alpha$ ~ if ~ $e(\v X_i)\leq\alpha$.}$$
By the symmetry of $\vepsilon_i$ we conclude statement (a).
   
Taking the full expectation $E=E_{\v X_i}E_{\vepsilon_i}$ on the latter expression yields the statement.
%We split the expectation as follows apart:
%$$ E\left(I_>D_i\right) = E_{\v X_i}\left( E_{\vepsilon_i}(I_>)D_i\right) ~ . $$
%By introducing densities $i$ and $f$ of the distributions of 
%$I:=E_{\vepsilon_i}(I_>)$ and $D_i$ with respect to $\v X_i$, 
%respectively, we may write:
%\begin{eqnarray*}
%  E(I_>D_i)= E\left(I\hspace{1pt}D_i\right) &=& \int_{-\infty}^\infty i(x)f(x)dx \\
%  &\stackrel{(1)}{=}&  \int_0^\infty i(x)\left(f(-x)+f(x)\right)dx  \\
%  &\stackrel{(3)}{\leq}& \alpha \int_0^\infty \left(f(-x)+f(x)\right)dx \\
%  & = & \alpha\int_{-\infty}^\infty f(x)dx \\
%  &=& \alpha E(D_i) ~ ,
%\end{eqnarray*}
%where (1) holds because $I=E_{\vepsilon_i}\left(I_>\right)$ is symmetric to 
%the origin and (3) holds because $I$ is bounded by $\alpha$ and $f(-x)\leq f(x)$. 
%To see that the latter holds, assume we had $\nu=0$. Then by symmetry of 
%$\vepsilon_i$, $D_i$ would be symmetric as well, hence $f(-x)=f(x)$. 
%This shows the second inequality of (a). The first inequality follows analog:
%\begin{eqnarray*}
% E(I_>D_i)= E\left(I\hspace{1pt}D_i\right) &=& \int_{-\infty}^\infty i(x)f(x)dx \\
% &\stackrel{(1)}{=}&  \int_0^\infty i(x)\left(f(-x)+f(x)\right)dx  \\
% &\stackrel{(4)}{\geq}&  0  ,
%\end{eqnarray*}
%where (4) holds because $I\geq0$ and $f(-x)\leq f(x)$. 

(b) ~
We denote $I_\leq:=I_{\{\Vert \vepsilon_i-\v X_i\Vert\leq r\}}$ and 
$I_>:=I_{\{\Vert \vepsilon_i-\v X_i\Vert>r\}}$.
Since it holds
$$ E(I_\leq \epsilon_i) + E(I_> \epsilon_i) = 
E\left(\left(I_\leq+I_>\right)\epsilon_i\right) = E(\epsilon_i) = 0 ~ , $$
we conclude
$$ E(I_\leq \epsilon_i) = - E(I_> \epsilon_i) = E(I_>(-\epsilon_i))\stackrel{(1)}{\leq}
\alpha ~ ,$$
where (1) holds because $||\epsilon_i||\leq 1$ and by Ax.~\ref{axiom:defense}  we have 
$E(I_>)\leq\alpha$.

Furthermore $E(I_\leq\epsilon_i)\geq 0$ is clear.
%$$ E(I_\leq \epsilon_i) = - E(I_> \epsilon_i) 
%\leq E(I_>)\cdot\Vert E(\epsilon_i|I_>)\Vert \stackrel{(5)}{\leq}\alpha ~ ,$$
%where $\epsilon_i|I>$ denotes the random variable $\epsilon_i$ conditioned on 
%$I_>$ and (5) holds because by Ax.~\ref{axiom:defense} we have $E(I_>)\leq\alpha$ and 
%$\Vert \epsilon_i\Vert \leq 1$. This shows the second inequality of (b). 
%The first inequality follows analog to (a) by pulling the expectation apart
% and integrating $I_>$ over the symmetric density of $\epsilon_i$.

(c) ~ The proof of (c) is analogous to that of (a) and (b).
\end{proof}

\begin{proof}\textbf{of Th.~\ref{th:limited-main}}

(a) ~ By Ax.~\ref{axiom:defense} we have 
\begin{equation}\label{eq:max}
 D_{i+1} = \max\left(0,D_i + \frac{1}{n}\left(B_i\left(f(\v X_i)-\v X_i\right) 
    + (1-B_i)I_{\{\Vert \vepsilon_i-\v X_i\Vert\leq r\}}\left(\vepsilon_i 
    - \v X_i\right)\right)\cdot\v a\right) ~ .
\end{equation}
By Prop.~\ref{prop:defense-optattack} an optimal attack strategy
can be defined by
$$\  f(x) = x + \v a ~ .$$
Inserting the latter equation into Eq.~\eqref{eq:max}, using $D_i\stackrel{\rm Def.}{=}\v X_i \cdot \v a$, 
and taking the expectation, we have
\begin{equation}\label{eq:exp-max}
 E(D_{i+1}) = E\left(\max\left(0,D_i + \frac{1}{n}\left(B_i
    + (1-B_i)I_{\{\Vert \vepsilon_i-\v X_i\Vert\leq r\}}\left(\epsilon_i 
    -  D_i\right)\right)\right)\right) ~ ,
\end{equation}
denoting $\epsilon_i=\vepsilon_i\cdot\v a$.
By the symmetry of $\vepsilon_i$ the expectation can be moved inside the maximum, hence 
the latter equation can be rewritten as
\begin{eqnarray}\label{eq:aux-bound}
  E(D_{i+1}) & \leq &  \left(1-\frac{1-\nu}{n}\right)E( D_i) + \frac{\nu}{n}  \\
  &  & + \frac{1-\nu}{n}\left( E\left( I_{\{\Vert \epsilon_i 
  -  X_i \Vert > r\}}D_i\right) + E\left(I_{\{\Vert \epsilon_i 
  -  X_i \Vert \leq r\}}\epsilon_i \right)\right) \nonumber ~ .
\end{eqnarray}
Inserting the inequalities (a) and (b) of Lemma~\ref{lemma:aux-bound} into the above equation results in:
\begin{eqnarray*}
  E(D_{i+1}) &\leq& \left(1-\frac{1-\nu}{n}\right)E( D_i) + \frac{\nu}{n} 
  + \frac{1-\nu}{n}\left(\alpha E(D_i)+\alpha\right) \\
             & = & \left(1-\frac{(1-\nu)(1-\alpha)}{n}\right)E( D_i) 
             + \frac{\nu+\alpha(1-\nu)}{n} ~ .
\end{eqnarray*}
By  the formula of the geometric series, i.e., Prop.~\ref{prop:geom-series},
we have
\begin{equation}\label{eq:exp-upper}
  E(D_{i+1})\leq(1-c_i)\frac{\nu+\alpha(1-\nu)}{(1-\nu)(1-\alpha)}
\end{equation}
where $c_i=\left(1-\frac{(1-\nu)(1-\alpha)}{n}\right)^i$.
Moreover we have
\begin{equation}
  E(D_{i+1})\geq(1-b_i)\frac{\nu}{1-\nu} ~ ,
\end{equation}
where $b_i=\left(1-\frac{1-\nu}{n}\right)^i$, by analogous reasoning. 
In a sketch we show that by starting at Eq.~\eqref{eq:exp-max}, and subsequently 
applying Jensen's inequality, the lower bounds of
Lemma~\ref{lemma:aux-bound} and the formula of the geometric series. Since $b_i\leq c_i$
we conclude
\begin{equation}\label{eq:exp-lower}
  E(D_{i+1})\geq(1-c_i)\frac{\nu}{1-\nu} ~ .
\end{equation}

% (b). In the following we use the notation of the previous lemma, i.e.
% $I_\leq:=I_{\{\Vert \vepsilon_i-\v X_i\Vert\leq r\}}$ and 
% $I_>:=I_{\{\Vert \vepsilon_i-\v X_i\Vert>r\}}$.
% For the proof of part (b) we need the following preliminary notes: ~
%(1) Any random variable $B$ only taking binary values is invariant 
%under squaring, i.e. $B^2=B$, furthermore $B(1-B)=0$. ~
%(2) $E(I_>D_i^2)\leq \alpha D_i^2$ and $E(I_\leq \epsilon_i D_i)
%     \leq\alpha E(D_i)$. ~
%(3) $I_\leq \epsilon_i^2\leq 1$.

%Proof of (1)-(3): Statement (1) is clear. We will use this result for the indicator
%and bernoulli variables considered here. Both formulas in (2) are 
%easily derived by pulling the expectation $E$ apart into two integrals 
%$E=E_{X_i}E_{\vepsilon_i}$ and applying Lemma~\ref{lemma:aux-bound}(a),(b) ($q=2$). Statement (3) follows from $I_\leq \leq 1$ together with 
%$0\leq||\epsilon||^2\leq 1$.

(b) ~ Rearranging terms in Eq.~\eqref{eq:max}, we have
\begin{eqnarray*}
 D_{i+1} &\leq& \max\left(0,\left(1-\frac{1-B_i}{n}\right)D_i + \frac{B_i}{n} + 
    \frac{1-B_i}{n}I_{\{\Vert \vepsilon_i-\v X_i\Vert\leq r\}}\vepsilon_i \right.\\
    && \left. + \frac{1-B_i}{n}I_{\{\Vert \vepsilon_i-\v X_i\Vert> r\}}D_i 
    \right)
\end{eqnarray*}
Squaring the latter equation at both sides and using that $D_i$, $I_{\{\Vert \vepsilon_i-\v X_i\Vert\leq r\}}$, and
$I_{\{\Vert \vepsilon_i-\v X_i\Vert> r\}}$ are binary-valued, yields
\begin{eqnarray*}
  D_{i+1}^2 \leq \left(1-\frac{1-B_i}{n}\left(2-\frac{1}{n}\right))\right) D_i^2 
  + 2\frac{B_i}{n}D_i 
  + \left(\frac{1-B_i}{n}\left(2-\frac{1}{n}\right)\right)I_{\{\Vert \vepsilon_i-\v X_i\Vert> r\}}D_i \\
  + 2\frac{1-B_i}{n}\left(1-\frac{1}{n}\right)I_{\{\Vert \vepsilon_i-\v X_i\Vert\leq r\}} \epsilon_iD_i
  +\frac{1-B_i}{n^2}I_{\{\Vert \vepsilon_i-\v X_i\Vert\leq r\}}\epsilon_i^2 +\frac{B_i}{n^2} ~.
\end{eqnarray*}

Taking expectation on the above equation, by Lemma~\ref{lemma:aux-bound}, we have
\begin{eqnarray*}
  E(D_{i+1}^2) &\leq& \left(1-\frac{1-\nu}{n}\left(2-\frac{1}{n})(1-\alpha)\right)\right)E(D_i^2) \\
  && +2\left(\frac{\nu}{n}+\alpha\frac{1-\nu}{n}\left(1-\frac{1}{n}\right)\right)E(D_i) 
   + \frac{\nu+(1-\nu)E(\epsilon_i^2)}{n^2} ~ .
\end{eqnarray*}

We are now in an equivalent situation as in the proof of Th.~\ref{prop:limited-optattack}, right after Eq.~\eqref{eq:2nd-momentaux}. 
Similary, we insert the result of (a) into the above equation, obtaining
\begin{eqnarray*}
  \lefteqn{E(D_{i+1}^2) \leq \left(1-\frac{1-\nu}{n}\left(2-\frac{1}{n})(1-\alpha)\right)\right)E(D_i^2)} \\
   && \hspace{1.5cm}+ 2\left(\frac{\nu}{n}+\alpha\frac{1-\nu}{n}\left(1-\frac{1}{n}\right)\right) 
       (1-c_i)\frac{\nu+\alpha(1-\nu)}{(1-\nu)(1-\alpha)}
  + \frac{\nu+(1-\nu)E(\epsilon_i^2)}{n^2} \\
  && \hspace{1.1cm}\leq \left(1-\frac{1-\nu}{n}\left(2-\frac{1}{n}\right)\left(1-\alpha\right)\right)E(D_i^2)
  + 2(1-c_i)\frac{(\nu + \alpha(1-\nu))^2}{n(1-\nu)(1-\alpha)} \\
  && \hspace{1.5cm}+ \frac{\nu+(1-\nu)E(\epsilon_i^2)}{n^2} \\
\end{eqnarray*}

By the formula of the geometric series we obtain 
\begin{eqnarray}\label{eq:1st-moment}
  E(D_i^2)&\leq& \left(2(1-c_i)\frac{(\nu + \alpha(1-\nu))^2}{n(1-\nu)(1-\alpha)}
  +\frac{\nu+(1-\nu)E(\epsilon_i^2)}{n^2}\right)
  \frac{1-d_i}{\frac{1-\nu}{n}(2-\frac{1}{n})(1-\alpha)} \nonumber\\
  &\leq& \frac{(1-c_i)(1-d_i)(\nu+\alpha(1-\nu))^2}{(1-\frac{1}{2n})(1-\nu)^2(1-\alpha)^2}
  +\frac{(1-d_i)(\nu+(1-\nu)E(\epsilon_i^2))}{(2n-1)(1-\nu)(1-\alpha)}  ~ ,
\end{eqnarray}
where $d_i=\left(1-\frac{1-\nu}{n}(2-\frac{1}{n})(1-\alpha)\right)^i$.
We finally conclude

\begin{eqnarray*}
  \lefteqn{\text{Var}(D_i)=E(D_i^2)-\left(E(D_i)\right)^2} \\
  &\stackrel{\eqref{eq:exp-lower},\eqref{eq:1st-moment}}{\leq}&
  \frac{(1-c_i)(1-d_i)(\nu+\alpha(1-\nu))^2}{(1-\frac{1}{2n})(1-\nu)^2(1-\alpha)^2}
  +\frac{(1-d_i)(\nu+(1-\nu)E(\epsilon_i^2))}{(2n-1)(1-\nu)(1-\alpha)} - (1-c_i)^2\frac{\nu^2}{(1-\nu)^2} \\
  &\stackrel{(1)}{\leq}& \gamma_i \frac{\nu^2}{(1-\alpha)^2(1-\nu)^2} + \rho(\alpha) + \delta_n
\end{eqnarray*}
defining
$\gamma_i=c_i-d_i$, ~ 
$\rho(\alpha)=\alpha\frac{(1-c_i)(1-d_i)(2\nu(1-\alpha)+\alpha)}{(1-\frac{1}{2n})(1-\nu)^2(1-\alpha)^2}$, ~ 
and $\delta_n=\frac{(1-d_i)(\nu+(1-\nu)E(\epsilon_i^2))}{(2n-1)(1-\nu)(1-\alpha)}$,
where (1) can be verified employing some algebra and using the auxiliary formula 
Eq.~\eqref{eq:aux-form}, which holds for all $0<d_i<c_i<1$. This completes the proof of (b). 

Statements (c) and (d) are easily derived from (a) and (b) by noting
hat $0\leq c_i<1$, $c_i\rightarrow 1$ for $i\rightarrow\infty$ and
$\delta(n)\rightarrow 0$ for $n\rightarrow\infty$. This completes 
the proof of the theorem.
\end{proof}

\bibliography{tr}

\end{document}